\documentclass[twoside,11pt]{article}

\usepackage{arxiv}

\usepackage[english]{babel}
\usepackage{amsmath,amssymb,bbm}

\usepackage{xspace}

\usepackage{xcolor}

\newcommand*{\colorboxed}{}
\def\colorboxed#1#{%
  \colorboxedAux{#1}%
}
\newcommand*{\colorboxedAux}[3]{%
  \begingroup%
    \colorlet{cb@saved}{.}%
    \color#1{#2}%
    \boxed{%
      \color{cb@saved}%
      #3%
    }%
  \endgroup%
}

\newcommand\Tstrut{\rule{0pt}{2.6ex}}         
\newcommand\Bstrut{\rule[-0.9ex]{0pt}{0pt}}   

\newcommand{\annotation}[1]{}

\newcommand{\TODOref}[1]{\annotation{\textcolor{red}{[ref]}}}

\newcommand{\rtrl}{\textsc{rtrl}\xspace}
\newcommand{\bptt}{\textsc{bptt}\xspace}
\newcommand{\uoro}{\textsc{uoro}\xspace}
\newcommand{\preuoro}{\textsc{preuoro}\xspace}
\newcommand{\reinforce}{\textsc{reinforce}\xspace}
\newcommand{\rnn}{\textsc{rnn}\xspace}
\newcommand{\rnns}{\rnn{}s\xspace}
\newcommand{\gir}{\textsc{gir}\xspace}

\newcommand{\jac}[2]{\smash{\mathcal{J}^{#1}_{#2}}}

\newcommand{\dldw}{\jac{L}\theta}
\newcommand{\dltdw}[1]{\jac{L_{#1}}{\theta}}
\newcommand{\dldwt}[1]{\jac{L}{\theta_{#1}}}
\newcommand{\dltdws}[2]{\jac{L_{#1}}{\theta_{#2}}}
\newcommand{\dltdwt}[1]{\jac{L_{#1}}{\theta_{#1}}}

\newcommand{\dldht}[1]{\jac{L}{h_{#1}}}
\newcommand{\dltdht}[1]{\jac{L_{#1}}{h_{#1}}}
\newcommand{\dltdhs}[2]{\jac{L_{#1}}{h_{#2}}}

\newcommand{\dhtdhs}[2]{\jac{h_{#1}}{h_{#2}}}

\newcommand{\dhtdw}[1]{\jac{h_{#1}}{\theta}}
\newcommand{\dhtdwt}[1]{\jac{h_{#1}}{\theta_{#1}}}
\newcommand{\dhtdws}[2]{\jac{h_{#1}}{\theta_{#2}}}

\newcommand{\dltdzt}[1]{\jac{L_{#1}}{z_{#1}}}
\newcommand{\dltdzs}[2]{\jac{L_{#1}}{z_{#2}}}

\newcommand{\dhtdzt}[1]{\jac{h_{#1}}{z_{#1}}}
\newcommand{\dhtdzs}[2]{\jac{h_{#1}}{z_{#2}}}

\newcommand{\dldzt}[1]{\jac{L}{z_{#1}}}
\newcommand{\dztdwt}[1]{\jac{z_{#1}}{\theta_{#1}}}

\newcommand{\ones}{\vec{1}}
\newcommand{\onehot}[1]{e_{#1}}
\newcommand{\indicator}[1]{\mathbbm{1}_{#1}}

\newcommand{\inlinefrac}[2]{\nicefrac{#1}{#2}}
\newcommand{\superfrac}[2]{\nicefrac{#1}{#2}}

\newcommand{\htilde}{\smash[t]{\tilde{h}}}
\newcommand{\wtilde}{\smash[t]{\tilde{w}}}

\newcommand{\bigO}{\mathcal{O}}

\DeclareMathOperator{\diag}{diag}

\let\originalleft\left
\let\originalright\right
\renewcommand{\left}{\mathopen{}\mathclose\bgroup\originalleft}
\renewcommand{\right}{\aftergroup\egroup\originalright}

\usepackage[utf8]{inputenc} 
\usepackage[T1]{fontenc}    
\usepackage{hyperref}       
\usepackage{url}            
\usepackage{booktabs}       
\usepackage{amsfonts}       
\usepackage{nicefrac}       
\usepackage{microtype}      
\usepackage{mathtools}

\usepackage{physics}

\newtheorem{claim}[theorem]{Claim}
\newtheorem{observation}[theorem]{Observation}

\newcommand{\E}{\mathbb{E}}
\DeclareMathOperator{\Var}{Var}
\DeclareMathOperator{\Cov}{Cov}
\DeclareMathOperator{\vectorized}{vec}

\newcommand{\thetitle}{On the Variance of Unbiased Online Recurrent Optimization}

\firstpageno{1}

\begin{document}

\title{\thetitle}

\author{%
\name Tim Cooijmans\thanks{Work partially carried out at DeepMind} \email cooijmat@mila.quebec \\
\addr Mila, Universit\'e de Montr\'eal \\
6666 St-Urbain, \#200 \\
Montreal, QC H2S 3H1, Canada
\AND
\name James Martens \email jamesmartens@google.com \\
\addr DeepMind \\
6 Pancras Square \\
London, N1C 4AG, United Kingdom
}


\maketitle

\begin{abstract}
  The recently proposed Unbiased Online Recurrent Optimization (\uoro) algorithm \citep{tallec2018unbiased} uses an unbiased approximation of \rtrl to achieve fully online gradient-based learning in \rnn{}s.  In this work we analyze the variance of the gradient estimate computed by \uoro, and propose several possible changes to the method which reduce this variance both in theory and practice. We also contribute significantly to the theoretical and intuitive understanding of \uoro (and its existing variance reduction technique), and demonstrate a fundamental connection between its gradient estimate and the one that would be computed by \reinforce if small amounts of noise were added to the \rnn's hidden units.
\end{abstract}

\begin{keywords}
recurrent neural networks, credit assignment, automatic differentiation
\end{keywords}

\section{Introduction}









All learning algorithms are driven by some form of credit assignment---identification of the causal effect of past actions on a learning signal \citep{minsky1961steps,sutton1984temporal}.
This enables agents to learn from experience by amplifying behaviors that lead to success, and attenuating behaviors that lead to failure.
The problem of performing efficient and precise credit assignment, especially in temporal agents, is a central one in artificial intelligence.

Knowledge of the inner workings of the agent can simplify the problem considerably, as we can trace responsibility for the agent's decisions back to its parameters.
In this work, we consider credit assignment in recurrent neural networks \citep[\rnns;][]{elman1990finding,hochreiter1997long}, where the differentiability of the learning signal with respect to past hidden units allows us to assign credit using derivatives. But even with this structure, \emph{online} credit assignment across long or indefinite stretches of time remains a largely unsolved problem.

Typically, differentiation occurs by Backpropagation Through Time \citep[\bptt;][]{rumelhart1986learning, werbos1990backpropagation}, which requires a ``forward pass'' in which the network is evaluated for a length of time, followed by a ``backwards pass'' in which gradient with respect to the model's parameters is computed.  This is impractical for very long sequences,  and a common trick is to ``truncate'' the backwards pass after some fixed number of iterations \citep{williams1990efficient}.
As a consequence, parameter updates are infrequent, expensive, and limited in the range of temporal dependencies
they reflect.

\bptt's more natural dual, Real-Time Recurrent Learning \citep[\rtrl;][]{williams1989learning},
carries gradient information forward rather than backward.
It runs alongside the model and provides parameters updates at every time step.
To do so, however, it must retain a large matrix relating the model's internal state to its parameters.
Even when this matrix can be stored at all, updating it is prohibitively expensive.
Various approximations to \rtrl have been proposed \citep[e.g.][]{mak1999improvement} in order to obtain cheaper gradient estimates
at the cost of reducing their accuracy.

In this paper we consider Unbiased Online Recurrent Optimization \citep[\uoro;][]{ollivier2015training, tallec2018unbiased}, an unbiased stochastic approximation to \rtrl that compresses the gradient information through random projections. We analyze the variance of the \uoro gradient estimator, relate it to other gradient estimators, and propose various modifications to it that reduce its variance both in theory and practice.


\section{Outline of the Paper}

We begin with a detailed discussion of the relationship and tradeoffs between \rtrl and \bptt in Section \ref{sec:autodiff}.
Before narrowing our focus to approximations to \rtrl, we briefly review
other approaches to online credit assignment in Section \ref{sec:related_work}.
We then contribute a novel and arguably more intuitive derivation of the \uoro algorithm in Section \ref{sec:uoro}.

In Sections \ref{sec:analysis} and \ref{sec:variance_reduction} we give our main contribution in the form of a thorough analysis of \uoro and the variance it incurs., and derive a new variance reduction method based on this analysis.
Sections \ref{sec:gir_limitations} and \ref{sec:gir_objective} discuss limitations of the variance reduction scheme of \citet{tallec2018unbiased},
and in Section \ref{sec:generalized_recursions} propose to augment its scalar coefficients with matrix-valued transformations.
We develop a framework for analysis of \uoro-style estimators in
Sections \ref{sec:total_gradient_estimate} and \ref{sec:total_variance},
which allows us to determine the total variance incurred when
accumulating consecutive gradient estimates over time.
Working within this framework,
we derive a formula for matrices that gives the optimal variance reduction subject to certain structural constraints (Section \ref{sec:optimizing_q}).
We evaluate our theory in a tractable empirical setting in Section \ref{sec:q_experiments},
and explore avenues toward a practical algorithm in Section \ref{sec:practical_considerations}.

Section~\ref{sec:preuoro} introduces a variant of \uoro that avoids one of its two levels of approximation.
It exploits the fact that gradients with respect to weight matrices are naturally rank-one.
We show this reduces the variance by a factor on the order of the number of hidden units, at the
cost of increasing computation time by the same factor.

Finally, we study the relationship between \uoro and \reinforce \citep{williams1992simple} in Section~\ref{sec:reinforce}.
The analysis uncovers a close connection when \reinforce is used to train \rnns with perturbed hidden states.
We show that when this noise is annealed, the \reinforce estimator converges to the \uoro estimator plus an additional term that has expectation zero but unbounded variance.

\section{Automatic Differentiation in Recurrent Neural Networks} \label{sec:autodiff}

Recurrent Neural Networks \citep[\rnns;][]{elman1990finding,hochreiter1997long} are a general
class of nonlinear sequence models endowed with memory. Given a sequence of
input vectors $x_t$, and initial state vector
$h_0$, an \rnn's state evolves according to
\begin{align*}
  h_t &= F (h_{t-1}, x_t; \theta_t)
\end{align*}
where $F$ is an arbitrary continuously differentiable transition function parameterized by
$\theta_t$ that produces the next state $h_t$ given the previous state $h_{t-1}$
and the current observation $x_t$. Typically, $F$ will take the form of an
affine map followed by a nonlinear function:
\begin{align}
  a_t &= (\begin{array}{ccc}
    h_{t - 1}^{\top} & x_t^{\top} & 1
  \end{array})^{\top} \notag \\
  h_t &= f (W_t a_t) . \label{eqn:F_standard_form}
\end{align}
Here $f(\cdot)$ is the ``activation function'', which is assumed to be continuously differentiable (and is typically nonlinear and coordinate-wise), and $W_t$ is a square matrix parameter whose vectorization is $\theta_t$.

The defining feature of recurrent neural networks as compared to feed-forward
neural networks is the fact that their weights are tied over time. That is, we have
$\theta_t = \theta$. However, we will continue to distinguish the different $\theta_t$'s in the recurrence, as this allows us to refer to
individual ``applications'' of $\theta$ in the analysis (which will be useful later).

Although we will treat the sequence as finite, i.e. $1 \leqslant t \leqslant T$ for
some sequence length $T$, we are interested mainly in \emph{streaming} tasks for
which $T$ may as well be infinite.

At each time step $t$, we incur a loss $L_t$ which is some differentiable
function of $h_t$. In order to minimize the aggregate loss $L = \sum_{t =
1}^T L_t$ with respect to $\theta$, we require an estimate of its gradient
with respect to $\theta$. We will write $\mathcal{J}^y_x$ (or occasionally $\mathcal{J}_x(y)$)
for the Jacobian of $y$ with respect to $x$. We can express the gradient as a double sum over time that factorizes in two interesting ways:
\begin{equation}
\label{eqn:gradientfactorizations}
  \dldw = \sum^T_{t = 1} \sum^T_{s = 1} \dltdws{t}{s}
=
\underbrace{
\sum^T_{s = 1} \left( \sum^T_{t = s} \dltdhs{t}{s} \right) \dhtdwt{s}
}_{\text{reverse accumulation}}
=
\underbrace{
\sum^T_{t = 1} \dltdht{t} \left( \sum^t_{s=1} \dhtdws{t}{s} \right)
}_{\text{forward accumulation}}
\end{equation}

Each of the terms $\dltdws{t}{s}$ indicates how the use of the
parameter $\theta$ at time $s$ affected the loss at time $t$.
This is a double sum over time with $\bigO (T^2)$ terms, but
since future parameter applications do not affect past losses, we have
$\dltdws{t}{s} = 0$ for $s > t$.
Both factorizations exploit this triangular structure
and allow the gradient to be computed in $\bigO (T)$
by recursive accumulation.

By far the most popular strategy for breaking down this computation goes by
the name of Back-Propagation Through Time \citep[\bptt;][]{werbos1990backpropagation}.
It is an instance of what is known as \emph{reverse-mode accumulation} in
the autodifferentiation community,
and relies on the reverse factorization in Equation~\ref{eqn:gradientfactorizations}.
\bptt computes
gradients of total future loss $\dldht{t}$ with respect to states $h_t$ in reverse chronological order by the recursion
\begin{equation} \label{eqn:bptt}
\dldht{t} = \dldht{t + 1} \dhtdhs{t+1}{t} + \dltdht{t}.
\end{equation}
At each step, a term $\dldwt{t} = \dldht{t} \dhtdwt{t}$ of the gradient is accumulated.

Since the quantities $\dhtdhs{t+1}{t}$,
$\dltdht{t}$ and $\dhtdwt{t}$ generally depend on
$h_t$ and $L_t$, the use of \bptt in practice implies running the model forward
for $T$ steps to obtain the sequence of hidden states $h_t$ and losses $L_t$,
and subsequently running backward to compute the gradient.

Its converse, Real-Time Recurrent Learning \citep[\rtrl;][]{williams1989learning}, is an instance of \emph{forward-mode accumulation}.
It exploits the forward factorization of the gradient in Equation~\ref{eqn:gradientfactorizations},
computing Jacobians $\dhtdw{t}$ of hidden states $h_t$ with respect to past applications of the parameter $\theta$
recursively according to
\begin{equation} \label{eqn:rtrl}
\dhtdw{t} = \dhtdhs{t}{t-1} \dhtdw{t-1} + \dhtdwt{t}.
\end{equation}

What \rtrl provides over \bptt is that we can run it forward alongside our
model, and at each time-step $t$ update the model parameters $\theta$ immediately (using $\dltdw{t} = \dldht{t} \dhtdw{t}$), thus performing fully online learning. This is to be contrasted with \bptt, where we
must run the model forward for $T$ time-steps before we can make a parameter update, thus introducing a long delay between the reception of a learning signal $L_t$ and the
parameter update that takes it into account.

There is a caveat to the above, which is that as soon as we update our
parameter $\theta$, the Jacobian $\dhtdw{t}$ accumulated by \rtrl is no
longer quite correct, as it is based on previous values of $\theta$.
However, as argued by \citet{williams1995gradient} and \citet{ollivier2015training} this problem can be mostly ignored as long as the learning rate is small enough
in relation to the rate of the natural decay of the Jacobian (which occurs due to the vanishing gradient phenomenon).

The main drawback of \rtrl is that the accumulated quantity
$\dhtdw{t}$ is a large matrix. If the size of the parameters $\theta$ is
$\bigO (H^2)$ where $H$ is the hidden state size, then this matrix requires $\bigO
(H^3)$ space to store. This is typically much larger than \bptt's $\bigO (T H)$
space. Moreover, the \rtrl recursions involve propagating a matrix forward by
the matrix-matrix product $\dhtdhs{t}{t-1} \dhtdw{t-1}$,
which takes $\bigO (H^4)$ time. \bptt on the other hand only propagates a
vector through time at a cost of $\bigO (H^2)$. Although \rtrl frees us to grow $T$
and capture arbitrarily-long-term dependencies, the algorithm is grossly impractical
for models of even modest size.

\section{Other Approaches to Credit Assignment} \label{sec:related_work}


A number of techniques have been proposed to reduce the memory requirements of \bptt.
Storage of past hidden states may be traded for time by recomputing the states on demand,
in the extreme case resulting in a quadratic-time algorithm.
Better choices for this tradeoff are explored by \citet{chen2016training,gruslys2016memory}.
Reversible Recurrent Neural Networks \citep{mackay2018reversible,gomez2017reversible} allow the on-demand computation of past states to occur in reverse order, restoring the linear time complexity while limiting the model class.
Stochastic Attentive Backtracking \citep{ke2018sparse} sidesteps the storage requirements of backprop through long periods of time by retaining only a sparse subset of states in the distant past. This subset is selected based on an attention mechanism that is part of the model being trained. Gradient from future loss is propagated backwards to these states only through the attention connections.
Synthetic gradients \citep{jaderberg2017decoupled} approximates \bptt by use of a predictive model of
the total future gradient $\dldht{s}$, which is trained online based on \bptt.

Instead of transporting derivatives through time, we may assign credit by transporting value over time.
For example, actor-critic architectures \citep{konda2000actor,barto1983neuronlike} employ
Temporal Difference Learning \citep{sutton1988learning}
to obtain a predictive model of the total future loss.
By differentiation,
the estimated total future loss may be used to estimate
the total future gradient.
More commonly, such estimates are used directly as a proxy for the total future loss, or as a \reinforce baseline.
Along similar lines as our analysis of \reinforce in Section \ref{sec:reinforce}, we may interpret these methods as effectively differentiating the estimate in expectation.
\textsc{rudder} \citep{arjona2018rudder} redistributes the total loss $L$ over time, replacing the immediate losses $L_s$ by surrogates $L'_s$ determined by a process similar to backpropagation through a critic. These surrogates preserve the total loss but in an RL setting may better reflect the long-term impact of the action taken at time $s$.
Temporal Value Transport \citep{hung2018optimizing} relies on attention weights to determine which past time steps were relevant to which future time steps, and injects the estimated total future loss from the future time steps into the immediate loss for the associated past time steps.

\section{Unbiased Online Recurrent Optimization} \label{sec:uoro}

The recently proposed Unbiased Online Recurrent Optimization algorithm \citep[\uoro;][]{tallec2018unbiased}
and its predecessor NoBackTrack \citep{ollivier2015training} approximate \rtrl by maintaining a
rank-one estimate $\htilde_t \wtilde_t^{\top}$ of the Jacobian
$\dhtdw{t}$. We now briefly derive the basic algorithm. 

\subsection{Derivation} \label{sec:uoro_derivation}

First, we note that $\dhtdw{t}$ can be written as
$\dhtdw{t} = \sum_{s \leqslant t} \dhtdhs{t}{s} \dhtdwt{s}$.
We then perform a rank-one projection of each term in this sum using a random vector $\nu_s$ (which is chosen to satisfy $\mathbbm{E} [\nu_s \nu_s^{\top}] = I$). This gives us the estimator
\[
\dhtdw{t} \approx \sum_{s \leqslant t} \dhtdhs{t}{s} \nu_s \nu_s^{\top} \dhtdwt{s} .
\]
Unbiasedness follows from a simple application of linearity of expectation:
\[
\mathbb{E} \Bigl[
\sum_{s \leqslant t} \dhtdhs{t}{s} \nu_s \nu_s^{\top} \dhtdwt{s}
\Bigr]
= 
\sum_{s \leqslant t} \dhtdhs{t}{s} \mathbb{E} [ \nu_s \nu_s^\top ] \dhtdwt{s}
=
\sum_{s \leqslant t} \dhtdhs{t}{s} \dhtdwt{s}. \]
We will refer to this projection as the \emph{spatial projection} to distinguish it from the \emph{temporal projection} that is to follow.

It is interesting to note that $\dhtdhs{t}{s} \nu_s$ can be interpreted as a ``directional Jacobian'', which measures the instantaneous change in $h_t$ as a function of $h_s$'s movement along the direction $\nu_s$. Similarly $\nu_s^{\top} \dhtdwt{s}$ is essentially the gradient of $\nu_s^{\top} h_s$ with respect to $\theta_s$, and thus measures the instantaneous change of $h_s$ along the direction of $\nu_s$, as a function of the change in $\theta_s$. Thus the intuition behind this first approximation is that we are guessing the relevant direction of change in $h_s$ and performing the gradient computations only along that direction.

We can generalize the spatial projection from the standard \uoro method by projecting in the space of any cut vertex $z_s$ on the computational path from $\theta_s$ to $h_s$. For \uoro, $z_s \equiv
h_s$; other choices include $z_s \equiv \theta_s$ for projection in parameter
space, and $z_s \equiv W_s a_s$ for projection in preactivation space.
We will make extensive use of this choice in later Sections.

This gives the generalized estimator
\[
\dhtdw{t} \approx \sum_{s \leqslant t} \dhtdhs{t}{s} \dhtdzt{s} \nu_s \nu_s^\top \dztdwt{s},
\]
which is unbiased following a similar argument as before.

The random projections serve to reduce the large $\dhtdwt{s}$
matrix into the more manageable vector quantities $\dhtdzt{s} \nu_s$
and $\nu_s^{\top} \dztdwt{s}$. But because the sum of rank-one
matrices is not itself rank one, the resultant estimator will still be too expensive to maintain and update online.

In order to obtain a practical algorithm we make a second rank-one approximation, now across time instead of $z$-space. To this end we introduce random scalar coefficients $\tau_s$ satisfying $\mathbbm{E}[\tau_s \tau_r] = \delta_{s r}$ (where $\delta_{s r}$ is the Kronecker delta which is 1 if $s = r$ and 0 otherwise) and define the following rank-one estimator:
\[ \dhtdw{t} \approx \htilde_t \wtilde_t^{\top} \triangleq
\Bigl( \sum_{s \leqslant t} \tau_s \dhtdhs{t}{s} \dhtdzt{s} \nu_s \Bigr)
\Bigl( \sum_{s \leqslant t} \tau_s \nu_s^{\top} \dztdwt{s} \Bigr)
= \sum_{r \leqslant t} \sum_{s \leqslant t} \tau_s \tau_r
  \dhtdhs{t}{s} \dhtdzt{s} \nu_s \nu_r^{\top}
  \dztdwt{r}
 . \]
   
By linearity of expectation this is an unbiased estimate of the previous spatially projected estimator
$\sum_{r \leqslant t} \dhtdzs{t}{s} \nu_s \nu_s^{\top} \dztdwt{s} $,
and is thus also an unbiased estimator of $\dhtdw{t}$, although with potentially much higher variance.

Going forward we will assume that $\tau_s \sim \mathcal{U} \{ - 1, + 1 \}$ are iid random signs
and $\nu_s \sim \mathcal{N} (0, I)$ are iid standard normal vectors,
so that we may treat the product $\tau_s \nu_s$ as a single Gaussian-distributed random vector $u_s \sim \mathcal{N} (0, I)$, which will simplify our analysis. 

The two factors $\htilde_t$ and $\wtilde_t$ of the rank-one approximation are maintained by the following pair of recursions:
\begin{align}
\label{eqn:recursion}
\htilde_t &= \gamma_t \dhtdhs{t}{t-1} \htilde_{t - 1}
             + \beta_t \dhtdzt{t} u_t \nonumber \\
\wtilde_t^{\top} &= \gamma_t^{-1} \wtilde_{t - 1}^{\top}
                    + \beta_t^{-1} u_t^{\top}
                    \dztdwt{t} ,
\end{align}
with $\htilde_0, \wtilde_0$ initialized to zero vectors. Notably these recursions are similar in structure to that used by \rtrl to compute the exact Jacobian $\dhtdw{t}$ (c.f. Equation~\ref{eqn:rtrl}). As with the \rtrl equations, their validity follows from the fact that $\dhtdhs{t}{s} = \dhtdhs{t}{t-1} \dhtdhs{t-1}{t-2} \cdots \dhtdhs{s+1}{s}$.

In these recursions we have introduced coefficients $\gamma_t$ and $\beta_t$ to implement the variance reduction technique from  \citet{tallec2018unbiased,ollivier2015training}, which we will refer to as \emph{greedy iterative rescaling} (\gir). We will discuss \gir in detail in the next subsection.

Finally, at each step we estimate $\dltdw{t} = \dltdht{t} \dhtdw{t}$ using the estimator
$\dltdht{t} \htilde_t \wtilde_t^{\top}$. This is a small
deviation from the one given by \citet{tallec2018unbiased}, which
uses backpropagation to compute $\dltdwt{t}$ exactly, and the remaining part of the gradient, $\sum_{s < t} \dltdws{t}{s}$, is estimated as $\dltdhs{t}{t-1} \htilde_{t - 1} \wtilde_{t - 1}^{\top}$. Although our version has slightly higher variance, it is conceptually simpler.

The projected Jacobians that appear in Equation~\ref{eqn:recursion} can be computed
efficiently without explicitly handling the full Jacobians. Specifically,
$u_t^{\top} \dztdwt{t}$ can be computed by reverse-mode
differentiating $z_t$ with respect to $\theta_t$, and substituting $u_t^{\top}$
in place of the adjoint $\dldzt{t}$. By a similar trick, one
can compute $\dhtdhs{t}{t-1} \htilde_{t - 1}$ and
$\dhtdzt{t} u_t$ using forward-mode differentiation. The resulting algorithm has the same $\bigO (H^2)$ time complexity as
backpropagation through time, but its $\bigO (H^2)$ storage does not grow with
time.

\subsection{Greedy Iterative Rescaling} \label{sec:uoro_gir}

This subsection explains \gir
and the role of the coefficients $\gamma_t, \beta_t > 0$
in Equation~\ref{eqn:recursion}.

\newcommand{\crossone}{\colorboxed{red}{\dhtdhs{t}{t-1} \htilde_{t - 1} \nu_t^{\top} \dztdwt{t}}}
\newcommand{\crosstwo}{\colorboxed{blue}{\dhtdzt{t} \nu_t \wtilde_{t - 1}^{\top}}}

Whereas our above derivation of the algorithm introduced a temporal projection, \citet{ollivier2015training,tallec2018unbiased} interpret the
algorithm given by Equation~\ref{eqn:recursion} as implementing a \emph{series}
of projections.
Under this view, $\htilde_t \wtilde_t^\top$ is a rank-one
estimate of the rank-two matrix that is the sum of
the forwarded previous Jacobian estimate
$\dhtdhs{t}{t-1} \htilde_{t - 1} \wtilde_{t - 1}^{\top}$ and the approximate contribution
$\dhtdzt{t} u_t u_t^{\top} \dztdwt{t}$:
\begin{align*}
  \htilde_t \wtilde_t^{\top} =\; & (
    \gamma_t \dhtdhs{t}{t-1} \htilde_{t - 1}
    + \beta_t \dhtdzt{t} u_t
    )
    (
    \gamma_t^{- 1} \wtilde_{t - 1}^{\top}
    + \beta_t^{- 1} u_t^{\top} \dztdwt{t} 
    )\\
  =\; &
  \dhtdhs{t}{t-1} \htilde_{t - 1} \wtilde_{t - 1}^{\top} + \dhtdzt{t} u_t u_t^{\top} \dztdwt{t}
+ \tau_t \gamma_t \beta_t^{- 1} \crossone + \tau_t \beta_t \gamma_t^{- 1} \crosstwo.
\end{align*}

The temporal ``cross-terms''
$\tau_t \gamma_t \beta_t^{- 1} \crossone$ and
$\tau_t \beta_t \gamma_t^{- 1} \crosstwo$, which are zero in expectation (but contribute variance), constitute the error introduced in the transition from time $t-1$ to $t$. The coefficients $\gamma_t$ and $\beta_t$ provide an extra degree of freedom with which we can minimize this error. As shown by \citet{ollivier2015training}, the minimizers ensure the terms
$\gamma_t \dhtdhs{t}{t-1} \htilde_{t-1}, \beta_t \dhtdzt{t} u_t$
and their $\wtilde_t$ counterparts have small norm,
so that their contribution to the variance is small as well.

The total (trace) variance of $\htilde_t \wtilde_t^\top$ with respect to $\tau_t$ is given by the expected squared Frobenius norm $\norm{\cdot}^2_F$ of the error:
\[
\E_{\tau_t} \Bigl[ \Bigl\| \htilde_t \wtilde_t^\top - \E_{\tau_t} [ \htilde_t \wtilde_t^\top ] \Bigr\|^2_F \Bigr]
= \E_{\tau_t} \Bigl[ \Bigl\|
\tau_t \gamma_t \beta_t^{- 1} \crossone +
\tau_t \beta_t \gamma_t^{- 1} \crosstwo
\Bigr\|^2_F \Bigr].
\]
As the common sign $\tau_t$ does not affect the norm, this is simply
\[
\gamma_t^2 \beta_t^{- 2} \Bigl\| \crossone \Bigr\|^2_F
 \!\!+ \beta_t^2 \gamma_t^{- 2} \Bigl\| \crosstwo \Bigr\|^2_F
 \!\!+ 2 \Bigl< \crossone, \crosstwo \Bigr>_{\!\!F},
\]
where $\left< \cdot , \cdot \right>_F$
denotes the Frobenius inner product.

The coefficients $\gamma_t$ and $\beta_t$ affect the error through the single degree of freedom $\gamma_t^2 \beta_t^{- 2}$. By differentiation and use of the identity $\| x y^\top \|^2_F = \| x \|^2 \| y \|^2$ we find that the optimal choices satisfy
\[
\gamma_t^2 \beta_t^{- 2} \| \dhtdhs{t}{t-1} \htilde_{t - 1} \|^2
                         \| \nu_t^{\top} \dztdwt{t} \|^2
=
\beta_t^2 \gamma_t^{- 2} \| \dhtdzt{t} \nu_t \|^2
                         \| \wtilde_{t - 1} \|^2 .
\]
This includes the solution
$\gamma_t^2 = \inlinefrac{\| \wtilde_{t - 1} \|}{\| \dhtdhs{t}{t-1} \htilde_{t - 1} \|},
 \beta_t^2 = \inlinefrac{\| \nu_t^{\top} \dztdwt{t} \|}{\| \dhtdzt{t} \nu_t \|}$
 from \citet{ollivier2015training}. 
 
 Examining their use in Equation \ref{eqn:recursion} we can see that for this particular solution $\gamma_t$ plays the important role of
 contracting $\wtilde_t$,
 which would otherwise grow indefinitely (being a sum of independent random quantities). While division by $\gamma_t$ in the recursion for $\htilde_t$ causes an expansive effect, this is more than counteracted by the natural contractive property of the Jacobian $\dhtdhs{t}{t-1}$ (which is due to gradient vanishing in well-behaved \rnns).  Thus we can interpret the role of $\gamma_t$ as distributing this contraction evenly between $\smash{\htilde_t}$ and $\wtilde_t$, which limits the growth of both quantities and thus keeps the variance of their product under control. A formal treatment of the growth of variance over time is given by \citet{masse2017around}.

\section{Variance Analysis} \label{sec:analysis}

In this section we analyze the variance behavior of \uoro-style algorithms.
We first discuss limitations of the \gir variance reduction scheme
discussed in Section \ref{sec:uoro_gir},
namely that it is greedy (Section \ref{sec:gir_limitations})
and derives from a somewhat inappropriate objective (Section \ref{sec:gir_objective}).
We then generalize the algorithm and develop a more holistic theoretical framework for its analysis (Sections \ref{sec:generalized_recursions} through \ref{sec:total_variance}).

\subsection{Greedy Iterative Rescaling is Greedy} \label{sec:gir_limitations}

In Section~\ref{sec:uoro_gir} we discussed how \gir can be interpreted as
minimizing the variance of a rank-one estimate $\htilde_t \wtilde_t^\top$
of a rank-two matrix $\dhtdhs{t}{t-1} \htilde_{t-1} \wtilde_{t-1}^\top + \dhtdzt{t} \nu_t \nu_t^\top \dztdwt{t}$ (which is a stochastic approximation that occurs at each step in \uoro).
Here we unify this sequence of approximations into a single temporal rank-one estimation
(as introduced in Section~\ref{sec:uoro_derivation}),
which helps us reveal the inherent limitations of \gir.

Recall that the \uoro recursions (Equation \ref{eqn:recursion}) maintain past contributions in the form of sums $\htilde_t$ and $\wtilde_t$,
and at each step \gir applies respective scaling factors $\gamma_{t+1}$ and $\gamma_{t+1}^{-1}$ (resp.) to these sums. This gives rise to an overall scaling $\alpha_s^{(t)} = \beta_s \gamma_{s+1} \gamma_{s+2} \dots \gamma_t$ (and similarly $(\alpha_s^{(t)})^{-1}$) of contributions made at time step $s$ and propagated forward through time step $t$.
We can write the estimates $\htilde_t \wtilde_t^\top$ produced by \uoro in terms of $\alpha_s^{(t)}$ as follows:
\[
\dhtdw{t} \approx
\htilde_t \wtilde_t^\top
=
\Bigl( \sum_{s \leqslant t} \alpha_s^{(t)} \dhtdzs{t}{s} u_s \Bigr)
\Bigl( \sum_{r \leqslant t} \frac{1}{\alpha_r^{(t)}} u_r^\top \dztdwt{r} \Bigr)
= \sum_{r \leqslant t} \sum_{s \leqslant t}
\frac{\alpha_s^{(t)}}{\alpha_r^{(t)}}
\tau_s \tau_r
   \dhtdhs{t}{s} \dhtdzt{s}
   \nu_s \nu_r^\top
   \dztdwt{r}.
\]
Note that each such estimate is but one element in a \emph{sequence} of estimates.
In the next section, we will establish a notion of the variance
for this sequence, so that we may speak meaningfully about its minimization.
For now, we will consider the minimization of the variance of $\htilde_t \wtilde_t^\top$ at each time step $t$ as an independent problem,
with independent decision variables $\alpha_s^{(t)}$.
The optimal coefficients given by
$(\alpha_s^{(t)})^2 = \inlinefrac{\| \nu_s^\top \dztdwt{s} \|}
                           {\| \dhtdzs{t}{s} \nu_s \|}$
(derived in Appendix \ref{sec:gir_limitations_variance}) minimize the variance of $\htilde_t \wtilde_t^\top$ with respect to $\tau_s$.

This solution is generally different from that of \gir,
which is constrained to have the form $\alpha_s^{(t+1)} = \gamma_{t+1} \alpha_s^{(t)}$ for $s \leqslant t$ (where $\gamma_{t+1}$ is independent of $s$).
This relationship between $\alpha_s^{(t+1)}$ and $\alpha_s^{(t)}$
breaks the independence of consecutive variance minimization problems,
and therefore the resulting coefficients cannot in general be optimal for all $t$.

We can see this by writing the optimal coefficients $\alpha_s^{(t+1)}$ for $s \leqslant t$
that minimize the variance of $\htilde_{t+1} \wtilde_{t+1}^\top$
in terms of the coefficients $\alpha_s^{(t)}$
that minimize the variance of $\htilde_t \wtilde_t^\top$:
\begin{align*}
(\alpha_s^{(t+1)})^2
=\;& \frac{\| \nu_s^\top \dztdwt{s} \|}{\| \dhtdzs{t+1}{s} \nu_s \|}
= \frac{\| \nu_s^\top \dztdwt{s} \|}{\| \dhtdzs{t}{s} \nu_s \|} \frac{\| \dhtdzs{t}{s} \nu_s \|}{\| \dhtdzs{t+1}{s} \nu_s \|}
\\=\;& (\alpha_s^{(t)})^2 \frac{\| \dhtdzs{t}{s} \nu_s \|}{\| \dhtdzs{t+1}{s} \nu_s \|}
= (\alpha_s^{(t)})^2 \left\| \dhtdhs{t+1}{t} \frac{\dhtdzs{t}{s} \nu_s}{\| \dhtdzs{t}{s} \nu_s \|} \right\|^{-1}.
\end{align*}
We see that in order to minimize the variance of $\htilde_{t+1} \wtilde_{t+1}^\top$
given coefficients $\alpha_s^{(t)}$ that minimize the variance of $\htilde_{t} \wtilde_t^\top$,
we should \emph{divide} each contribution $\alpha_s^{(t)} \dhtdzs{t}{s} \nu_s$ by
the square root of its contraction due to forward-propagation through $\dhtdhs{t+1}{t}$,
and \emph{multiply} each $(\alpha_s^{(t)})^{-1} \nu_s^\top \dztdwt{s}$ by the same factor.
Crucially, this factor depends on $s$ and therefore cannot be expressed by \gir,
which is constrained to rescale all past contributions by a constant factor $y_{t+1}$ independent of $s$.
This is true of any algorithm that maintains past contributions in a reduced form
such as $\htilde_t, \wtilde_t$.

\subsection{Greedy Iterative Rescaling Optimizes an Inappropriate Objective} \label{sec:gir_objective}

In the previous subsection, we saw a sense in which \gir is greedy: its ability to
minimize the variance of $\htilde_t \wtilde_t^\top$ is hampered by its own past decisions.
To see this, we took a holistic view of the sequence of variance minimization problems
solved by \gir, and showed that the choice of coefficients $\gamma_s, \beta_s$ at time $s$
constrains the choice of future coefficients.
Here we take a further step back, and argue that the variance of
$\htilde_t \wtilde_t^\top$
is not the right objective in light of the downstream application of these estimates.

The Jacobian estimates
$\htilde_t \wtilde_t^\top \approx \dhtdw{t}$
are used to determine a sequence of \emph{gradient} estimates
$\dltdht{t} \htilde_t \wtilde_t^\top \approx \dltdw{t}$,
which are accumulated by a gradient descent process.
We argue that the quantity of interest is
the variance of the \textbf{total gradient estimate}
$\sum_{t \leqslant T} \dltdht{t} \htilde_t \wtilde_t^\top \approx \dldw$
incurred during $T$ steps of optimization
(which estimates the \textbf{total gradient} $\dldw$).

Since consecutive gradient contributions depend largely on the same stochastic quantities,
the variance of this sum is not simply the sum of the individual variances.
Hence even if we \emph{could} independently minimize the variances of the Jacobian estimates,
doing so is not equivalent to minimizing the variance of the total gradient estimate.

\subsection{Generalized Recursions} \label{sec:generalized_recursions}

Before proceeding with the variance computation we will generalize the \uoro recursions by replacing the $\gamma_t$ and $\beta_t$ coefficients by an invertible matrix $Q_t$ as follows:
\begin{align}
  \label{eqn:generalized_recursion}
  \htilde_t &= \dhtdhs{t}{t-1} \htilde_{t - 1} + \dhtdzt{t} Q_t u_t \nonumber \\
  \wtilde_t^{\top} &= \wtilde_{t - 1}^{\top} + u_t^{\top} Q_t^{-1} \dztdwt{t}
\end{align}
$Q_t$ can be interpreted as modifying the covariance of the noise vector $u_t$ (although differently for either recursion). 
Analogously to the standard \uoro recursions, our generalized recursions compute the following sums:
\begin{align*}
    \htilde_t = \sum_{s \leqslant t} \dhtdhs{t}{s} \dhtdzt{s} Q_s u_s
    \mbox{\quad\quad and \quad\quad}
    \wtilde_t = \sum_{s \leqslant t} u_s^{\top} Q_s^{-1} \dztdwt{s}.
\end{align*}

We can view $Q_s$ as a matrix-valued generalization of the \gir coefficients, with equivalence when $Q_s = \beta_s \gamma_{s+1} \gamma_{s+2} \dots \gamma_{T} I$.
The extra degrees of freedom allow more fine-grained control over the norms of cross-terms,\footnote{By ``cross-term'' we mean a term that appears in the expanded sum which is zero in expectation but contributes variance.}
as can be seen when we expand both the temporal \emph{and} the spatial projections in
the estimator $\htilde_t \wtilde_t^\top$:
\[
\dhtdw{t} \approx \htilde_t \wtilde_t^\top = \sum_{r \leqslant t} \sum_{s \leqslant t} \sum_{i j k l} \dhtdzs{t}{r i} (Q_r)_{i k} u_{r k} u_{s l} (Q_s^{-1})_{l j} \jac{z_{s j}}{\theta_s}
\]
Each term's scaling depends not just on temporal indices $r, s$ 
but now also on the indices $i, j$ of units.
As we shall see, \emph{in expectation},
terms where both the temporal indices $r = s$ and units $i = j$ correspond remain unaffected,
and it is only the undesired cross-terms for which $r \neq s$ or $i \neq j$ that are affected.

\citet{tallec2018unbiased} hint at a related approach which would correspond to choosing
$Q_s = \alpha_s \diag(q_s)$ to be diagonal matrices.
However, they derive their choice
$q_{s i}^2 \propto \| \jac{z_{s i}}{\theta_s} \| / \| \jac{h_s}{z_{s i}} \|$
by optimizing the norms of only temporally corresponding terms for which $r = s$,
and ignoring temporal cross terms $r \neq s$ which make up the bulk of the error.
We instead consider a class of $Q_s$ matrices that is not constrained to be diagonal,
and whose value minimizes a measure of variance that is more relevant to the optimization process.

Thus our recursion in Equation~\ref{eqn:generalized_recursion} is
a strict generalization of the \uoro recursion in Equation~\ref{eqn:recursion}.
The $Q_s$ matrices can express a broad class of variance reduction mechanisms,
including \gir.
That said, our analysis of this system will be limited to cases where the $Q_s$
are independent of the noise vectors $u_t$ for all $s, t$.
Notably, this precludes \gir because of its complex nonlinear interaction with the noise.

\subsection{A Simple Expression for the Gradient Estimate} \label{sec:total_gradient_estimate}

In this subsection we will derive a simple expression for the gradient estimate $\dltdht{t} \htilde_t \wtilde_t^\top$ which will prove useful in our subsequent computations.

To reduce visual clutter we define the following notational aliases, which we will make heavy use of throughout the rest of the manuscript:
\begin{equation*}
    b^{(t)}_s = \dltdzs{t}{s} \mbox{\quad\quad and \quad\quad} J_s = \dztdwt{s}.
\end{equation*}

First, we observe that that $\mathbbm{1}_{s \leqslant t} b^{(t)}_s = b^{(t)}_s$, as derivatives of past losses with respect to future activations are zero. Next we observe that
\begin{align*}
    \dltdht{t} \htilde_t = \sum_{s \leqslant t} \dltdht{t} \dhtdhs{t}{s} \dhtdzt{s} Q_s u_s = \sum_{s \leqslant t} b^{(t) \top}_s Q_s u_s .
\end{align*}
Given these observations we may express the estimate of each gradient contribution $\dltdw{t}$ as
\begin{align*}
  \dltdht{t} \htilde_t \wtilde_t^{\top}
=\;& \Bigl( \sum_{s \leqslant t} b^{(t) \top}_s Q_s u_s \Bigr)
     \Bigl( \sum_{s \leqslant t} u_s^{\top} Q_s^{- 1} J_s \Bigr) \\
=\;& \Bigl( \sum_{s \leqslant T} \mathbbm{1}_{s \leqslant t} b^{(t) \top}_s Q_s u_s \Bigr)
     \Bigl( \sum_{s \leqslant T} \mathbbm{1}_{s \leqslant t} u_s^{\top} Q_s^{- 1} J_s \Bigr) \\
=\;& \Bigl( \sum_{s \leqslant T} b^{(t) \top}_s Q_s u_s \Bigr)
     \Bigl( \sum_{s \leqslant T} \mathbbm{1}_{s \leqslant t} u_s^{\top} Q_s^{- 1} J_s \Bigr) \\
=\; & b^{(t) \top} Q u u^{\top} Q^{- 1} S^{(t)} J ,
\end{align*}
where in the last step we have:
\begin{itemize}
    \item consolidated the temporal and spatial projections by concatenating the $b^{(t)}_s$ into a single vector $b^{(t)}$, and the noise vectors $u_s$ into a single vector $u$,
    \item stacked the $J_s$'s into the matrix $J$,
    \item defined $Q$ to be the block-diagonal matrix $\diag(Q_1, Q_2, \dots, Q_T)$, and
    \item introduced the ``truncated identity matrix'' $S^{(t)}$ with diagonal blocks $S_s^{(t)} = \mathbbm{1}_{s \leqslant t} I$.
\end{itemize}
Finally, the total gradient estimate is given by
\begin{align} \label{eqn:total_gradient_estimate}
\sum_{t \leqslant T} \dltdht{t} \htilde_t \wtilde_t^{\top} = \sum_{t \leqslant T} b^{(t) \top} Q u u^{\top} Q^{- 1} S^{(t)} J .
\end{align}

The $S^{(t)}$ matrix accounts for the fact that at time $t$ of the algorithm,
contributions $\dztdwt{s}$ from future steps $s > t$ are
not included in $\wtilde_t^{\top}$. Omitting this matrix would introduce
terms that are zero in expectation and hence would not bias the total gradient
estimate, but they would still contribute to the variance of the estimator (to a degree which would adversely affect the usefulness of our subsequent analysis).

It is easy to see that this estimator is unbiased as long as
$\mathbb{E} \left[ Q u u^\top Q^{-1} \right] = I$.
This can happen, for example, when $Q$ and $u$ are independent with $\mathbb{E}[u u^\top] = I$. We will focus our analysis on this case.

\subsection{Computing the Variance of the Total Gradient Estimate}
\label{sec:total_variance}

In this section we derive the variance of the total gradient estimate.
We assume that $Q$ is independent of $u$, so that we may use the general results from Appendix~\ref{sec:secondmoment}.

By bilinearity, the covariance matrix of the total gradient estimate is
\begin{align*}
\Var\Bigl[ \sum_{t \leqslant T} \dltdht{t} \htilde_t \wtilde_t^{\top} \Bigr]
= \sum_{t \leqslant T} \sum_{s \leqslant T} \Cov \Bigl[ \dltdht{t} \htilde_t \wtilde_t^{\top}, \dltdht{s} \htilde_s \wtilde_s^{\top} \Bigr].
\end{align*}
Combining this with the identity $\dltdht{t} \htilde_t \wtilde_t^{\top} = b^{(t) \top} Q u u^{\top} Q^{- 1} S^{(t)} J$ from the previous subsection and applying Corollary \ref{cor:variance} (with $\kappa = 0$) yields the following expression for the same quantity:
\begin{equation*}
  \sum_{s \leqslant T} \sum_{t \leqslant T}
  \tr(b^{(s)} b^{(t) \top} Q Q^{\top})
  J^{\top} S^{(s)} (Q Q^{\top})^{- 1} S^{(t)} J 
  + J^{\top} b^{(s)} b^{(t) \top} J .
\end{equation*}
Corollary \ref{cor:variance} also yields the following expression for the \emph{total variance}\footnote{We define the ``total variance'' to be the trace of the covariance matrix.} of the total gradient estimate:
\begin{equation*} \label{eqn:total_gradient_variance_trace}
  \sum_{s \leqslant T} \sum_{t \leqslant T}
  \tr( b^{(s)} b^{(t) \top} Q Q^{\top} )
  \tr(J^{\top} S^{(s)} (Q Q^{\top})^{- 1} S^{(t)} J )
  + \tr( J^{\top} b^{(s)} b^{(t) \top} J ) .
\end{equation*}

\section{Variance Reduction} \label{sec:variance_reduction}

We now turn to the problem of reducing the variance given in Equation \ref{eqn:total_gradient_variance_trace}.
In Sections \ref{sec:optimizing_q} through \ref{sec:practical_considerations} we develop
an improved (though as yet impractical) variance reduction scheme.
Finally, we evaluate our theory in Section \ref{sec:q_experiments}.

\subsection[Optimizing Q subject to restrictions on its form]{Optimizing $Q$ subject to restrictions on its form}
\label{sec:optimizing_q}

Denote by $V(Q)$ the part of the total variance (Equation \ref{eqn:total_gradient_variance_trace}) that depends on $Q$. Making use of the cyclic
property of the trace, and the fact that $Q$ is block-diagonal, we can write this as
\begin{equation} \label{eqn:general_vq}
V(Q) = \sum_{s \leqslant T} \sum_{t \leqslant T}
\trace \Bigl( \sum_{r \leqslant T} b_r^{(s)} b_r^{(t) \top} Q_r Q_r^{\top} \Bigr)
\trace \Bigl( \sum_{r \leqslant T} S_r^{(t)} J_r J_r^{\top} S_r^{(s)} (Q_r Q_r^{\top})^{- 1} \Bigr) .
\end{equation}

We wish to optimize $V(Q)$ with respect to $Q$ in a way that leads to a practical online algorithm.
To this end, we require that $Q_s$ be of the form $Q_s = \alpha_s Q_0$, with $\alpha_s$ a scalar and $Q_0$ a constant matrix.
This restriction makes sense from a practical standpoint;
we envision an algorithm that maintains a statistical estimate of the optimal value of $Q_0$.
The stationarity assumption enables us to amortize over time both the sample complexity of obtaining this estimate, and the computational cost associated with inverting it.

We furthermore assume projection occurs in preactivation space, that is, $z_r \equiv W_r a_r$.
This assumption gives $J_r = \dhtdzt{r} = I \otimes a_r^\top$, which is a convenient algebraic structure to work with.

Even given this restricted form we cannot find the jointly optimal solution for $Q_0$ and $\alpha$. Instead, we will consider optimizing $Q_0$ while holding the $\alpha_s$'s fixed, and vice versa.

\subsubsection[Optimizing alpha coefficients given Q0]{Optimizing $\alpha_s$ coefficients given $Q_0$}
\label{sec:optimizing_alpha}

Let us first simplify the expression for $V(Q)$. Given the restricted form $Q_s = \alpha_s Q_0$ we may write
\begin{equation} \label{eqn:vq_alpha_c}
V(Q) = \sum_{r \leqslant T} \sum_{q \leqslant T} \frac{\alpha_r^2}{\alpha_q^2} C_{q r} ,
\end{equation}
where we have collected the factors that do not depend on $\alpha$ into the matrix $C$ with elements
\begin{align} \label{eqn:matrix_c}
C_{q r} &= \sum_{s \leqslant T} \sum_{t \leqslant T}
\trace \Bigl( b_r^{(s)} b_r^{(t) \top} Q_0 Q_0^{\top} \Bigr)
\trace \Bigl( S_q^{(t)} J_q J_q^{\top} S_q^{(s)} (Q_0 Q_0^{\top})^{- 1} \Bigr) \notag
\\ &=
\trace \Bigl( \sum_{s = q}^T \sum_{t = q}^T b_r^{(s)} b_r^{(t) \top} Q_0 Q_0^{\top} \Bigr)
\trace \Bigl( J_q J_q^{\top} (Q_0 Q_0^{\top})^{- 1} \Bigr) \notag
\\ &= \Bigl\| \sum_{t = q}^T b_r^{(t) \top} Q_0 \Bigr\|^2 \Bigl\| Q_0^{-1} J_q \Bigr\|^2_F.
\end{align}

Now we wish to solve
\begin{align} \label{eqn:alphas_optimization_problem}
\alpha^\star =
    \operatornamewithlimits{argmin}_{\alpha > 0}
    \sum_{r \leqslant T} \sum_{q \leqslant T} \frac{\alpha_r^2}{\alpha_q^2} C_{q r}.
\end{align}
The optimization problem considered here differs from that given in Section \ref{sec:gir_limitations}.
Although the objective considered there
can similarly be written in terms of a matrix like $C$,
that matrix would have rank one (see Appendix \ref{sec:gir_limitations_variance}).
This difference is a consequence of $V(Q)$ being the variance of the \emph{total} gradient estimate rather than
that of a single contribution $\dltdht{t} \htilde_t \wtilde_t^\top$.
In particular, the rank-one property is lost due to our inclusion of the $S^{(t)}$ matrix that discards noncausal terms (see Section \ref{sec:total_gradient_estimate}).

We analyze the problem in Appendix \ref{sec:alphas_optimization_algorithm},
and find that it is an instance of matrix equilibration \citep[see e.g.][for a review]{idel2016review},
for which no closed-form solution is known.
Instead, we give a second-order
steepest-descent update rule that solves
for $\alpha$ numerically,
which we use in our experiments. (Empirically, first-order updates routinely get stuck in cycles on this problem.)


However, solving Equation \ref{eqn:alphas_optimization_problem} directly does not lead to a practical algorithm.
Along the lines of the discussion in Section \ref{sec:gir_limitations}, any algorithm that maintains past contributions as a single sum must take $\alpha_s$ to be $\beta_s \gamma_{s+1} \gamma_{s+2} \dots \gamma_T$ for some coefficient sequences $\{\beta_s\}$ and $\{\gamma_s\}$.
In principle, if $C$ were known upfront, one could choose $\beta_s = \alpha_s^\star$ with $\gamma_s = 1$,
and hence this parameterization appears to be degenerate.
However, $C$ is not known; it depends on gradients $b^{(t)}_r = \dltdzs{t}{r}$ and Jacobians $J_t = \dztdwt{t}$ from future time steps $t > s$.
In light of this, we can view $\beta_s$ as merely an estimate of $\alpha_s^\star$,
to be corrected by future $\gamma_t$'s as more information becomes available.

One way of formalizing this idea of ``incomplete information'' is as follows.
Suppose $C$ were the final element $C^{(T)}$ of a sequence of matrices $C^{(1)} \dots C^{(T)}$,
where each $C^{(s)}$ incorporates all ``information'' available up to time $s$.
Then a natural way to choose $\beta_s$ and $\gamma_s$ at time $s$ would be solve the following optimization problem based on $C^{(s)}$:
\begin{align} \label{eqn:greedy_alphas_optimization_problem}
\beta_s^\star, \gamma_s^\star =
    \operatornamewithlimits{argmin}_{\beta_s, \gamma_s}
    \operatornamewithlimits{min}_{\beta_{>s}, \gamma_{>s}}
    \sum_{r \leqslant T} \sum_{q \leqslant T}
    \frac{\alpha_r^2}{\alpha_q^2}
    C^{(s)}_{q r}.
\end{align}
Past coefficients $\beta_{<s}, \gamma_{<s}$ are known (and fixed),
and the unknown future coefficients $\beta_{>s}, \gamma_{>s}$ are estimated by the inner minimization.

In Appendix \ref{sec:greedy_alphas_example} we explore a natural choice for $C^{(s)}$ where future gradients/Jacobians are treated as though they were 0, which leads to formulas for the coefficients that are similar to \gir's, although not identical.
This approach can be improved by incorporating statistical predictions or estimates of unknown future information in $C^{(s)}$.
We leave further exploration of such schemes to future work.

\subsubsection[Optimizing Q0 given the alphas]{Optimizing $Q_0$ given the $\alpha_s$'s}
\label{sec:optimizing_q0}

Given our assumption that $z_r \equiv W_r a_r$ we have $J_r = I \otimes a^{\top}_r$ and $J_r
J_r^{\top} = (I \otimes a_r^{\top}) (I \otimes a_r) = (I \otimes a_r^{\top}
a_r) = \| a_r \|^2 I$. Thus,
\[ S_r^{(t)} J_r J_r^{\top} S_r^{(s)} (Q_r Q_r^{\top})^{- 1} =\mathbbm{1}_{r
   \leqslant t} \mathbbm{1}_{r \leqslant s} \| a_r \|^2 \alpha_r^{- 2} (Q_0
   Q_0^{\top})^{- 1} , \]
and $V (Q)$ becomes
\[
\sum_{s \leqslant T} \sum_{t \leqslant T}
\trace \Bigl( \sum_{r \leqslant T} \alpha_r^2 b_r^{(s)} b_r^{(t) \top} Q_0 Q_0^{\top} \Bigr)
\trace \Bigl( \sum_{r=1}^{\min (s, t)} \| a_r \|^2 \alpha_r^{- 2} (Q_0 Q_0^{\top})^{- 1} \Bigr)
. \]
Now we can move the scalar $\sum_{r=1}^{\min (s, t)} \| a_r \|^2
\alpha_r^{- 2}$ leftward and group the terms that depend on $s$ and $t$, giving
\begin{equation} \label{eqn:structured_vq}
V (Q) = \trace (B Q_0 Q_0^{\top}) \trace ((Q_0 Q_0^{\top})^{-1}) ,
\end{equation}
where
\begin{equation} \label{eqn:b_matrix}
B = \sum_{s \leqslant T} \sum_{t \leqslant T}
\Bigl( \sum_{q = 1}^{\min(s, t)}  \alpha_q^{- 2} \| a_q \|^2 \Bigr)
\Bigl( \sum_{r \leqslant T} \alpha_r^2 b_r^{(s)} b_r^{(t) \top} \Bigr)
 = \sum_{q \leqslant T} \sum_{r \leqslant T}
  \frac{\alpha_r^2}{\alpha_q^2}
  \| a_q \|^2
  \Bigl( \sum_{s = q}^{T} b_r^{(s)} \Bigr)
  \Bigl( \sum_{t = q}^{T} b_r^{(t)} \Bigr)^{\top}.
\end{equation}

The matrix $B$ is PSD (it is a sum of PSD matrices), and we will further assume it is
invertible. By Theorem~\ref{theorem:minimizer} (which is stated and proved in Appendix~\ref{sec:minimization}) any choice of $Q_0$ satisfying $\eta B Q_0 Q_0^{\top} = (Q_0 Q_0^{\top})^{- 1}$ for some constant $\eta > 0$ will be a global minimizer of $V(Q)$. One such choice is
\begin{align*}
    Q_0 = B^{-\superfrac{1}{4}} .
\end{align*}
This solution, or any other globally optimal one, gives us
\[ V (Q) = \trace (B^{\superfrac{1}{2}})^2 ,\]
where $\lambda$ is the vector of eigenvalues of $B^{1 / 2}$. We can
compare this to the variance attained by temporal scaling only ($Q_0 = I$):
\[ V (Q) = \trace (B) \trace (I) .\]
Writing $\trace (B^{1 / 2})^2 = (\ones^{\top} \lambda)^2$ and
$\trace (B^{}) \trace (I) = \| \ones \|^2 \| \lambda \|^2$,
where $\ones$ is the vector of ones and $\lambda$ is the vector of
eigenvalues of $B^{1 / 2}$, we have by the Cauchy-Schwarz inequality that
\begin{align*}
\trace (B^{\superfrac{1}{2}})^2 = (\ones^{\top} \lambda)^2 \leqslant \| \ones \|^2 \| \lambda \|^2 = \trace (B) \trace (I) .
\end{align*}
This approaches equality as $\lambda$ approaches a multiple of $\ones$, or in other words, as the spectrum of $B^{\superfrac{1}{2}}$ becomes flat. Conversely, the inequality will be more extreme when the spectrum is lopsided, indicating improved variance reduction when using $Q_0 = B^{-\superfrac{1}{4}}$ over the default choice $Q_0 = I$.

\subsubsection{Practical Considerations} \label{sec:practical_considerations}

In practice, the proposed choice of $Q_0$ requires computing
the $B$ matrix and its eigendecomposition.
Computing $B$ involves four levels of summations over time
and seemingly cannot be computed online.
However, we can estimate it using quantities similar to the
ones we use to estimate the gradient.
Appendix \ref{sec:estimating_B} derives the following unbiased estimator of $B$:
\[
B \approx \frac{1}{2} ( \tilde{m}_T \tilde{n}_T^\top + \tilde{n}_T \tilde{m}_T^\top )
\]
where $\tilde{m}_t$ is given by
\[
\sum_{s \leqslant t}
\Bigl( \sum_{q \leqslant s} \sigma_q \alpha_q^{-1} \norm{a_q} \Bigr)
\Bigl( \sum_{r \leqslant s} \tau_r \alpha_r b^{(s)\top}_r \nu_r \Bigr)
\Bigl( \sum_{r \leqslant s} \nu_r \Bigr)
\]
and $\tilde{n}_t$ is like $\tilde{m}_t$ except with spatial noise $\mu_r$ instead of and independent of $\nu_r$.
In these expressions, $\sigma, \mu$ are temporal and spatial noise vectors distributed identically to $\tau, \nu$.
This extra layer of stochastic approximation severely degrades the quality of the estimates.
Additionally, the estimator depends on unknown future
quantities, such as the total future gradient with respect to all time steps.
As detailed in Appendix \ref{sec:estimating_B},
we may compute intermediate estimates based on $\tilde{m}_t, \tilde{n}_t$ for $t < T$.
To the extent that $B$ is stationary,
a moving average of these intermediate estimates
can serve as a good approximation to $B$.

Empirically however, computing $Q_0$ based on this kind of estimator does not seem to improve optimization performance, due to its high variance. We leave a broader exploration of approximation algorithms for $B$ to future work, while noting that an estimator for $B$ need not be unbiased in order for us to obtain an unbiased estimate of the gradient. Indeed, \emph{any} invertible choice of $Q_0$ will result in an unbiased estimate of the gradient, as was shown in Section \ref{sec:total_gradient_estimate}. Unbiasedness may not even be a particularly desirable property for the $B$ estimator to have, compared to other reasonable-sounding properties such as positive-semidefiniteness.

Once we have our estimate $\hat{B}$ of $B$ and wish to compute its fourth root,
the $\bigO (H^3)$ cost of factorization could be amortized
by only performing it every so often
or maintaining the estimate in factored form.
It is often advisable to ``dampen'' or ``regularize'' the estimate by adding a multiple of the identity, i.e.
\[
Q_0 = (\hat{B} + \lambda I)^{\superfrac{1}{4}}
\]
where the hyperparameter $\lambda$ serves to control
the amount of trust placed in the estimate
by biasing it towards a flat eigenvalue spectrum (i.e. towards $Q_0 \propto I$).

\subsection{Variance Reduction Experiments} \label{sec:q_experiments}

\begin{figure}
    \centering
    \includegraphics[width=0.9\textwidth]{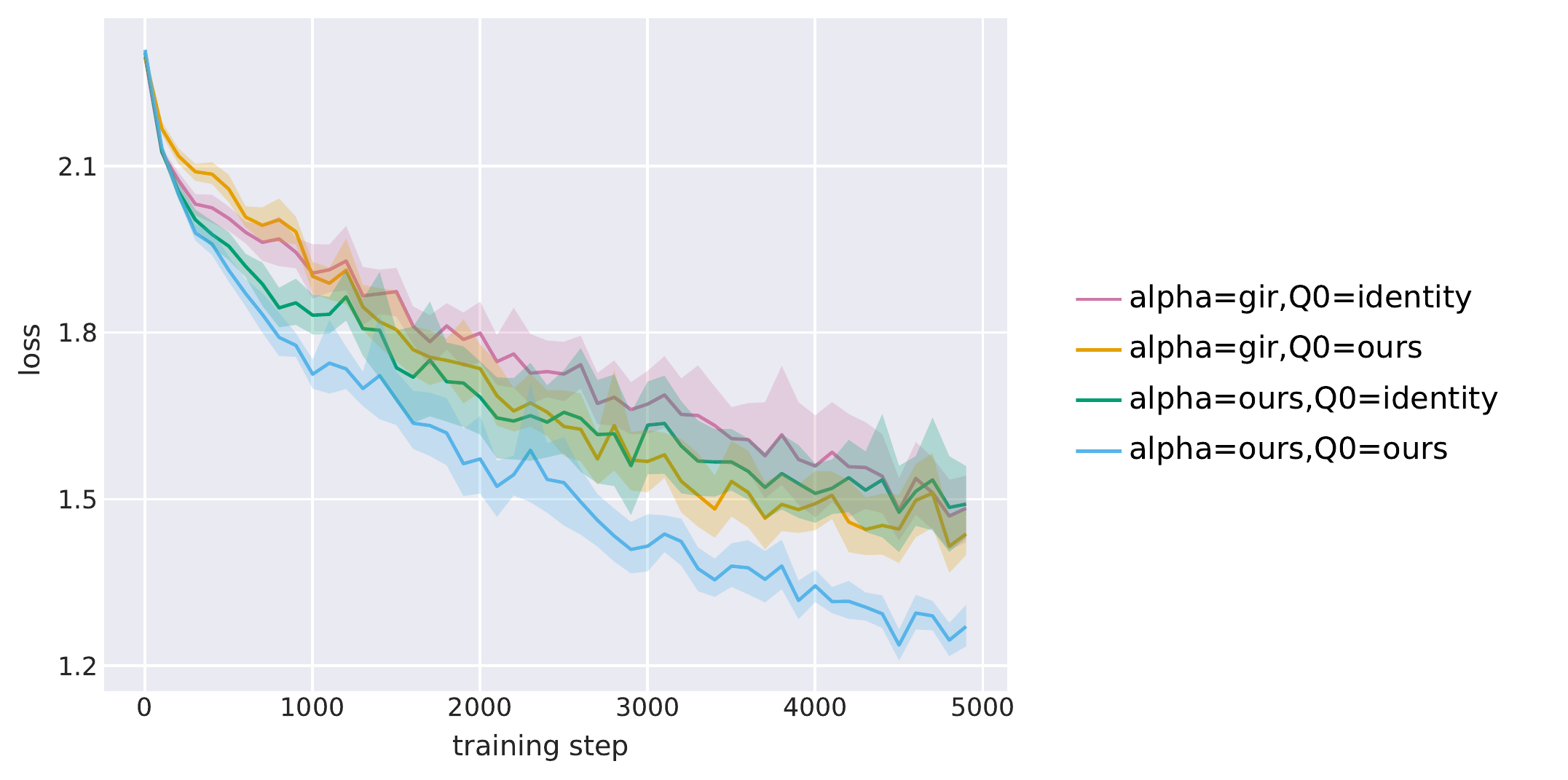}
    \caption{
    Training curves on the row-wise sequential \textsc{mnist} task.
    For each setting we have run 10 trials and plotted the mean of the classification loss and the 95\% confidence interval of the mean.
    For clarity of presentation, these curves have been aggressively smoothed by a median filter prior to the computation of their statistics.}
    \label{fig:seqmnist_curves}
\end{figure}

We empirically evaluate four settings for $Q_s = \alpha_s Q_0$
in a controlled setting based on the sequential \textsc{mnist} task \citep{le2015simple}.
We choose this task because it is episodic;
it gives us access to gradients $b^{(t)}_s$ and Jacobians $J_s$ 
for all $s, t$ by \bptt.
Thus we can compute the matrices $B$ and $C$ from Sections \ref{sec:optimizing_q0} and \ref{sec:optimizing_alpha} exactly.
In order to curb the cost of these computations,
we simplify the task to be row-by-row instead of pixel-by-pixel (i.e. $T = 28$ as opposed to $T = 784$).
Moreover, the model is tasked with classifying the digit
at every step rather than only at the end,
as otherwise $L_t = 0$ and therefore $b^{(t)}_s = 0$ for $t < T$,
trivializing the total gradient estimate (Equation \ref{eqn:total_gradient_estimate}).

For $\alpha_s$, we compare the \gir-style coefficients
\[
\gamma_s^2 = \frac{ \norm{\wtilde_{s-1}} }{ \norm{\dhtdhs{s}{s-1} \htilde_{s-1}} \vphantom{\tilde{h}} }
\mbox{\quad\quad and \quad\quad}
\beta_s^2 = \frac{ \norm{u_s^\top Q_0^{-1} \dztdwt{s}} }{ \norm{\dhtdzt{s} Q_0 u_s} }
\]
against the ones prescribed by our analysis.
In the latter case, we use the algorithm described in Appendix \ref{sec:alphas_optimization_algorithm}
to solve Equation \ref{eqn:alphas_optimization_problem} for $\alpha$.
Given $\alpha$, we derive a sequence of $\gamma, \beta$ coefficients
by setting $\gamma_s$ equal to the geometric average ratio of consecutive $\alpha_s$'s,
and solving for $\beta$ such that $\alpha_s = \beta_s \gamma_{s+1} \dots \gamma_T$ for all $s$.\footnote{
The simpler choice $\gamma_s = 1, \beta_s = \alpha_s$ may run into numerical issues but is otherwise equivalent,
as the distribution of the total scaling $\alpha$ across $\gamma, \beta$ does not affect the variance.}

For $Q_0$, we consider the naive choice $Q_0 = I$ as well as the solution $Q_0 = B^{-\superfrac{1}{4}}$
from Section \ref{sec:optimizing_q0}.
Recall that the optimal $Q_0$ depends on the choice of $\alpha$
and both choices of $\alpha$ depend on the choice of $Q_0$.
We break this circularity by maintaining an exponential moving average $\bar{B}$ of $B$
across episodes, which we use to compute $Q_0$ according to
\[
Q_0 = \Bigl( \bar{B} + \lambda \frac{ \trace (\bar{B}) }{ \trace (I) } I \Bigr)^{\superfrac{1}{4}},
\]
where the amount of damping/regularization is controlled by the hyperparameter $\lambda$.
Given $Q_0$, we compute $\alpha$ exactly,
process the episode and update the parameters
by the total gradient estimate (Equation \ref{eqn:total_gradient_estimate}).
At the end of the episode,
we compute $B$ exactly based on the $\alpha$ used in the episode,
average it across the minibatch,
and use the result to update $\bar{B}$.

The model consists of an \textsc{lstm} \citep{hochreiter1997long} with 50 hidden units.
At each step, the digit is classified by softmax regression based on the hidden state $h_t$.
As the classifier parameters do not affect $h_t$, their gradient is obtained by backprop.
The gradients are averaged across a minibatch of 50 examples and across the duration of each episode, before being passed to the Adam \citep{kingma2014adam} optimizer.
The settings of the learning rate, momentum and $\bar{B}$ decay and dampening
hyperparameters are detailed in Appendix \ref{sec:q_experiments_hyperparameters}.

Figure \ref{fig:seqmnist_curves} shows the training curves for each of the four configurations.
While there is a clear advantage to using both our proposed $\alpha$ and $Q_0$ choices,
that advantage appears to be lost when only one of the two is used.

\begin{figure}
    \centering
    \includegraphics[width=0.9\textwidth]{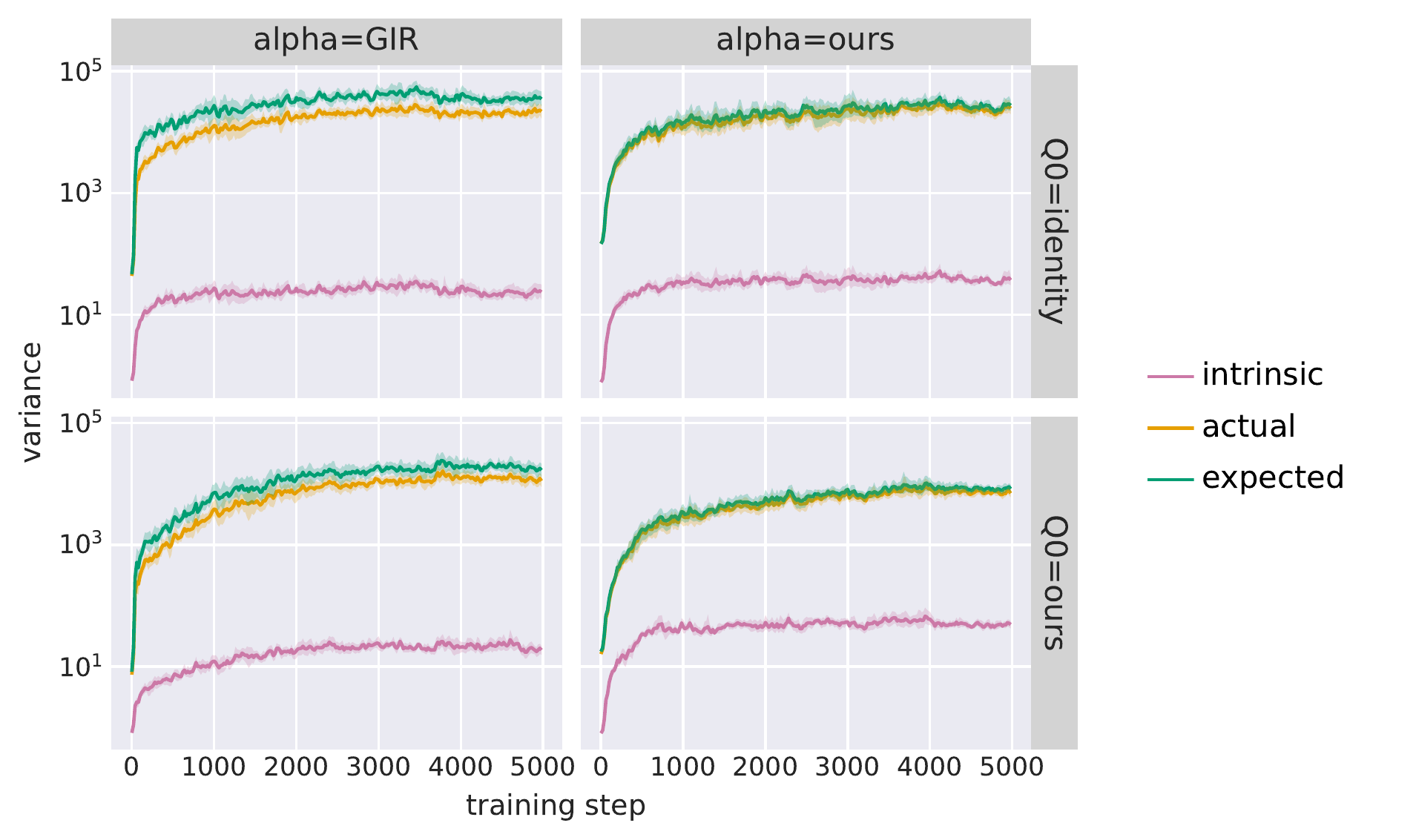}
    \caption{
    Theoretical predictions and empirical measurements
    of quantities contributing to total gradient variance.
    The ``intrinsic'' variance measures the expected norm of the total gradient $\dldw$,
    estimated by averaging across the minibatch.
    The ``expected'' variance is a theoretical prediction of $V(Q)$ according to Equation \ref{eqn:structured_vq}.
    The ``actual'' variance measures $V(Q)$ empirically by the expected norm of the total gradient estimate.
    }
    \label{fig:seqmnist_variances}
\end{figure}

In order to test our variance analysis,
we show in Figure \ref{fig:seqmnist_variances} predictions and measurements of several quantities that contribute to the variance,
recorded during optimization.
Recall from Section \ref{sec:total_variance} that the variance of the total
gradient estimate takes the form
\[
\Var \Bigl[ \sum_{t \leqslant T} \dltdht{t} \htilde_t \wtilde_t^{\top} \Bigr] = V(Q) + \bigl\| \dldw \bigr\|^2.
\]
The \textsf{actual} variance in Figure \ref{fig:seqmnist_variances} measures $V(Q)$ empirically by computing
\[
\E \Bigl[ \bigl\| \sum_{t \leqslant T} \dltdht{t} \htilde_t \wtilde_{t-1}^\top - \dldw \bigr\|^2 - \bigl\| \dldw \bigr\|^2 \Bigr],
\]
where the expectation is estimated by averaging across the minibatch.
The \textsf{intrinsic} variance is similarly computed as $\E \bigl[ \| \dldw \|^2 \bigr]$.
The \textsf{expected} variance measures
the theoretical prediction of $V(Q)$ by plugging the corresponding choice of $Q_0$ into Equation \ref{eqn:structured_vq}.

We see that the theoretical predictions of $V(Q)$ are correct when \textsf{alpha=ours},
but that they overestimate $V(Q)$ when \textsf{alpha=GIR}.
When we derived $V(Q)$ in Section \ref{sec:total_variance},
we started with the assumption that $Q$ and $u$ be independent;
this assumption is violated by the \gir coefficients, which depend on the noise $u$.
Finally, we see that our proposals indeed reduce the \textsf{actual} variance;
significantly so when both \textsf{Q0=ours, alpha=ours}.

We furthermore highlight in Figure \ref{fig:seqmnist_logalphas} the difference in behavior of the $\alpha$ coefficients under the four configurations.
The \gir coefficients appear to take on more extreme values, especially early on in training.
Presumably, poor initialization causes increased levels of gradient vanishing,
which subsequently causes $\gamma_s$ to be large in order to compensate.
However, when we combine the \gir coefficients with our choice of $Q_0$,
the effect is exacerbated.
This may be because the \gir coefficients and our $Q_0$ optimize for conflicting objectives.
Curiously, when both \textsf{Q0=ours, alpha=ours}, the relative ordering of the coefficients is reversed,
so that $\alpha_s < \alpha_t$ for $s < t$.

\begin{figure}
    \centering
    \includegraphics[width=0.9\textwidth]{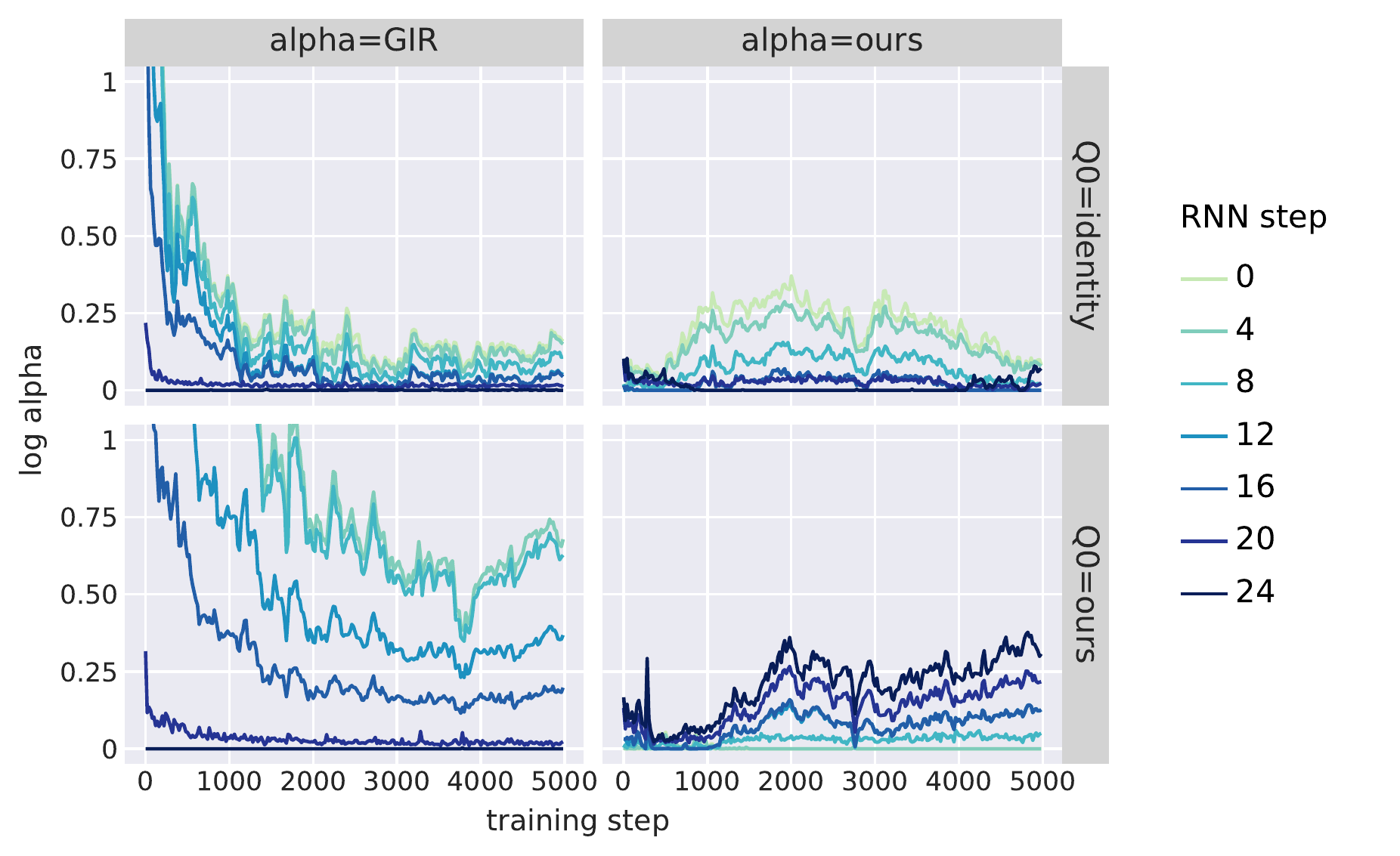}
    \caption{
    Evolution of $\log \alpha_s$ for some time steps $s$ as training proceeds.
    At each training step, the $\log \alpha_s$ are centered so that $\min_s \log \alpha_s = 0$;
    this eliminates irrelevant constant factors.
    }
    \label{fig:seqmnist_logalphas}
\end{figure}

\section{Projection in the Space of Preactivations} \label{sec:preuoro}

\newcommand{\dhdpre}{\mathcal{J}^{h_t}_{W_t a_t}}
\newcommand{\thisalgorithm}{\preuoro\xspace}
\newcommand{\thatalgorithm}{\uoro\xspace}

Recall from Section \ref{sec:uoro} how the spatial rank-one approximation breaks down the Jacobian
$\dhtdwt{t} = \dhtdzt{t} \nu_t \nu_t^{\top} \dztdwt{t}$
into more manageable quantities
$\dhtdzt{t} \nu_t$ and $\nu_t^{\top} \dztdwt{t}$ by
projecting in the space of some cut vertex $z_t$.
Assuming the transition function $F$ takes the form given in Equation~\ref{eqn:F_standard_form},
we observe that the Jacobian
can be factored as $\dhtdwt{t} = \dhdpre (I \otimes a_t^{\top})$ where $\otimes$ denotes the Kronecker product,
i.e. it is already rank-one. By choosing $z_t$ to be the
preactivations $W_t a_t$, we can avoid the projection, and we obtain
the following recursion:
\begin{align}
  \htilde_t &= \gamma_t \dhtdhs{t}{t-1} \htilde_{t - 1} + \beta_t \tau_t \dhdpre
   \label{eqn:preuoro_recursion} \\
  \wtilde_t &= \gamma_t^{-1} \wtilde_{t - 1} + \beta_t^{-1} \tau_t a_t \notag
\end{align}
The vector-valued $\htilde_t$ has been replaced by a matrix
$\htilde_t$, and the contributions $\dhdpre$ and $a_t$
are multiplied by scalar noise $\tau_s \sim \mathcal{N}(0,1)$ rather than projected down.
At each step,
the gradient contribution $\dltdw{t}$ is computed as $\vectorized (( \dltdht{t} \htilde_t)^{\top} \wtilde_t^{\top})$.
The \gir coefficients
\[
\gamma_t^2 = \inlinefrac{\|\wtilde_{t - 1}\|}{\| \dhtdhs{t}{t-1} \htilde_{t - 1}\|_F}
\mbox{,\quad\quad}
\beta_t^2 = \inlinefrac{\| a_t \|}{\|\dhdpre \|_F}
\]
can be derived like in Section~\ref{sec:uoro}.
We will refer to this variant of \thatalgorithm as ``\thisalgorithm''.
This algorithm has also been discovered by \citet{mujika2018approximating}.

Define $b^{(t)}_s = \dltdzs{t}{s}$, the gradient of the loss at time $t$ with respect to the projection variable at time $s$.
Then the total gradient $\dldw$ can be expressed as
\begin{equation*}
\dldw = \sum_{t \leqslant T} \dltdw{t}
  = \sum_{t \leqslant T} \sum_{s \leqslant t} b_s^{(t) \top} (I \otimes a_s^{\top})
  = \sum_{t \leqslant T} \sum_{s \leqslant t} \vectorized ( b_s^{(t)} a_s^\top ),
\end{equation*}
where $\vectorized$ is the vectorization operator that serializes its matrix argument into a row vector in row-major order.
We can express the total gradient estimate
as
\begin{align}
\vectorized \Bigl( \sum_{t \leqslant T} (\dltdht{t} \htilde_t)^{\top} \wtilde_t^{\top} \Bigr)
&= \vectorized \Bigl( \sum_{t \leqslant T}
\bigl( \sum_{s \leqslant t} \tau_s \alpha_s b_s^{(t)} \bigr)
\bigl( \sum_{r \leqslant t} \tau_r \alpha_r^{-1} a_r^{\top} \bigr)
\Bigr)
\notag
\\&= \vectorized \Bigl( \sum_{t \leqslant T}
\bar{B}^{(t) \top} \bar{Q} \tau \tau^{\top} \bar{Q}^{-1} \bar{S}^{(t)} \bar{J}
\Bigr) , \label{eqn:preuoro_total_gradient_estimate}
\end{align}
where we have defined the matrices
\[
\bar{B}^{(t) \top} = \begin{pmatrix}
b_1^{\smash[t]{(t)}} & \cdots & b_T^{\smash[t]{(t)}}
\end{pmatrix},
\bar{Q} = \diag(\alpha),
\bar{S}^{(t)}_{i j} = \delta_{i j} \indicator{i \geqslant t},
\bar{J} = \begin{pmatrix}
a_1 & \cdots & a_T
\end{pmatrix}^\top .
\]
that mirror similarly-named quantities from Section \ref{sec:total_gradient_estimate}.
The expression in Equation \ref{eqn:preuoro_total_gradient_estimate} is analogous to that in Equation \ref{eqn:total_gradient_estimate},
but with the crucial difference that no summation across space is involved.
Hence the noise vector $\tau$ has much smaller dimension $T$ rather than $T N$
(with $N$ being the dimension of the projection space).

We show in Appendix \ref{sec:preuoro_variance_comparison} that the variance contribution $V(Q)$ of \thisalgorithm can be written
\[
\sum_{s \leqslant T} \sum_{t \leqslant T}
\trace \bigl( \bar{B}^{(s)} \bar{B}^{(t)\top} \bar{Q} \bar{Q}^\top \bigr)
\trace \bigl( \bar{S}^{(t)} \bar{J} \bar{J}^\top \bar{S}^{(s)} (\bar{Q} \bar{Q}^\top)^{-1} \bigr)
\]
and the variance contribution $V(Q)$ of \thatalgorithm's total gradient estimate from Section \ref{sec:total_gradient_estimate}
(Equation \ref{eqn:total_gradient_estimate}) can be written:
\[
\sum_{s \leqslant T} \sum_{t \leqslant T}
\trace \bigl( \bar{B}^{(t)} Q_0 Q_0^\top \bar{B}^{(s)\top} \bar{Q} \bar{Q}^\top \bigr)
\trace \bigl( \bar{S}^{(t)} \bar{J} \bar{J}^\top \bar{S}^{(s)} (\bar{Q} \bar{Q}^\top)^{-1} \bigr)
\trace \bigl( (Q_0 Q_0^\top)^{-1} \bigr)
\]
The latter has an extra factor $\trace ((Q_0 Q_0^\top)^{-1})$.
If $Q_0 = I$, then this factor is equal to $tr(I)$.
Spatial projection thus causes the dominant term of the variance
to be multiplied by the dimension of the preactivations,
which typically ranges in the thousands.
Avoiding the spatial projection avoids this multiplication and
hence achieves drastically lower variance.

\begin{figure}
    \centering
    \includegraphics[width=0.95\textwidth]{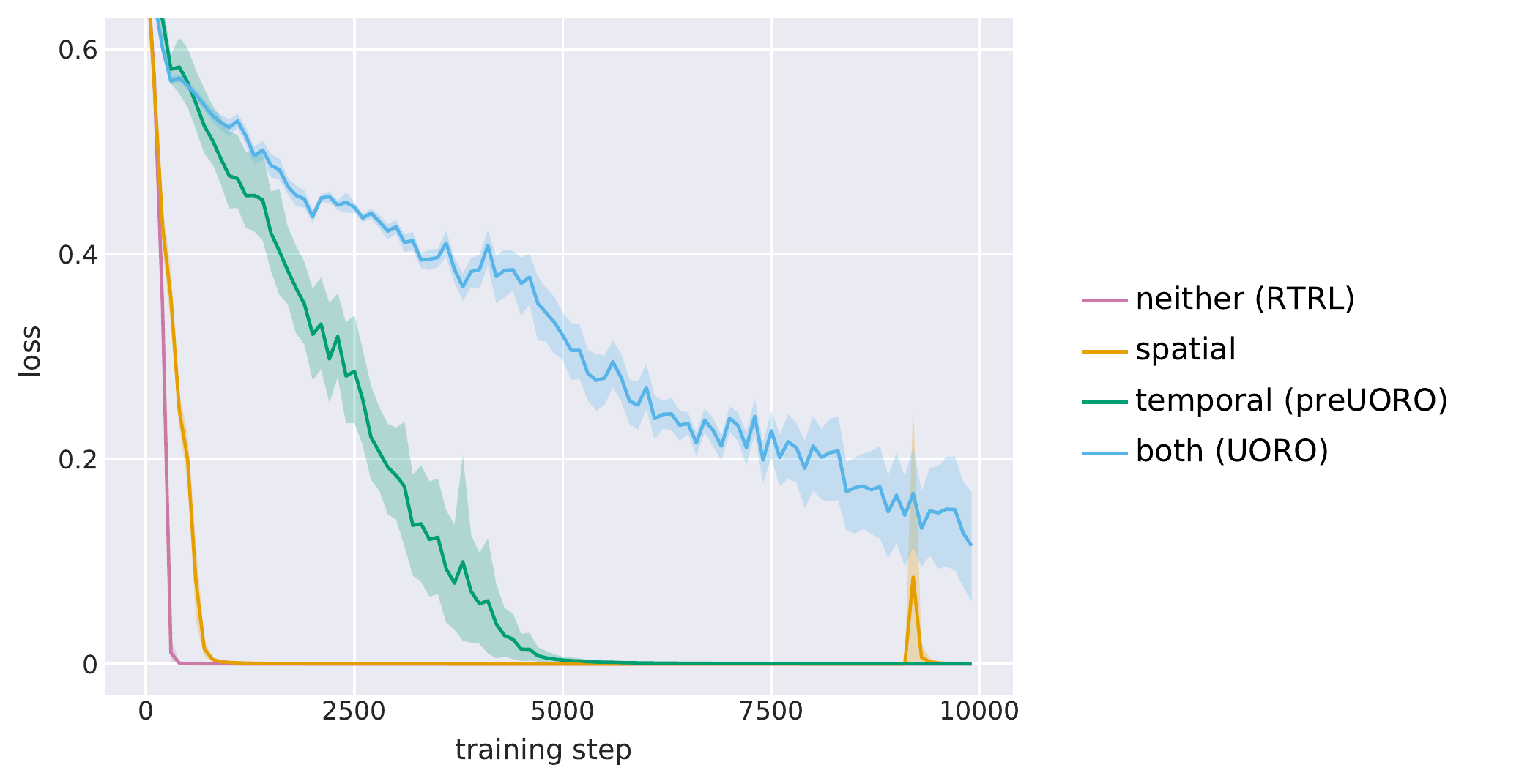}
    \caption{
    Training curves on the queue task showing interpolation
    between \rtrl and \uoro by ablation of the spatial and temporal approximations.
    ``\textsf{neither}'' denotes exact computation of the gradient using \rtrl,
    ``\textsf{spatial}'' denotes \rtrl with $\dhtdzt{t} \nu_t \nu_t^\top \dztdwt{t}$ standing in for $\dhtdwt{t}$,
    ``\textsf{temporal}'' denotes \thisalgorithm computed by Equation \ref{eqn:preuoro_recursion},
    ``\textsf{both}'' denotes \thatalgorithm.
    Where applicable, the cut vertex $z_t \equiv W_t a_t$ is the preactivations.
    }
    \label{fig:preuoro_ablation}
\end{figure}

Figure \ref{fig:preuoro_ablation} confirms the corresponding improvement in optimization performance.
This figure shows training curves of four variations on \rtrl:
\rtrl,
\rtrl plus spatial projection (\thatalgorithm minus temporal projection),
\thisalgorithm (\thatalgorithm minus spatial projection), and
\thatalgorithm which performs both spatial and temporal projection.
The task under consideration is the queue task,
in which the model is trained to emit its input stream with a delay.
Effectively, the model learns to implement a queue.

The model is similar to that described in \ref{sec:q_experiments}, except with 50 hidden units.
The model observes a random binary input stream and has to predict a binary output stream
that is equal to the input stream but with a delay of 4 time steps.
The $\dltdw{t}$ estimates are averaged across a minibatch of 100 examples,
and applied to the parameters by Adam \citep{kingma2014adam} with momentum 0.5 and learning rate set to
0.008 for ``\textsf{neither}'',
0.008 for ``\textsf{spatial}'',
0.0008 for ``\textsf{temporal}'',
0.002 for ``\textsf{both}'' (found by grid search).

The main drawback of this method is its computational complexity:
the algorithm involves propagating multiple vectors forward,
which increases the computation time by the same factor
$N$ that we removed from the variance.
The dominant operation is the matrix-matrix
multiplication $\dhtdhs{t}{t-1} \htilde_{t - 1}$, which has
computational cost $\bigO (N^3)$ (recall $N$ is the dimension of the projection space). This is better than \rtrl's
$\dhtdhs{t}{t-1} \dhtdw{t-1}$ which costs $\bigO
(N^4)$, but worse than \uoro and \bptt which propagate vectors at a cost of $\bigO
(N^2)$. The space complexity is $\bigO (N^2)$, which matches that of \uoro.

\section{\reinforce as Approximate Real-Time Recurrent Learning} \label{sec:reinforce}

In this section we show a fundamental connection between \reinforce~\citep{williams1992simple} and \uoro.
The \reinforce algorithm provides gradient estimates for systems with
stochastic transitions. It can also be used to train recurrent neural networks
if we artificially induce stochasticity by adding Gaussian noise to the hidden
states.
We will show that in this
setting, the \reinforce estimator is closely related to the \uoro estimator.

\reinforce aims to estimate the gradient of the expected loss
$\E_{\chi \sim p (\chi ; \theta)} [L (\chi)]$ which depends on the
parameter $\theta$ through some distribution $p (\chi ; \theta)$ over
stochastic context $\chi$ that determines the loss $L (\chi)$. Conceptually, $\chi = (\chi_t)$ is
the trajectory of the state of an agent and its external environment, and $\theta$ parameterizes a
stochastic policy over actions, which induces a distribution $p (\chi ;
\theta)$ on $\chi$. 

The gradient of the expected loss can be rewritten as an
expected gradient as follows:
\begin{align*}
  \nabla_{\theta} \E_{\chi \sim p (\chi ; \theta)} [L (\chi)]
  = \nabla_{\theta} \int L (\chi) p (\chi ; \theta) d \chi
  &= \int L (\chi) \nabla_{\theta} p (\chi ; \theta) d \chi \\
  &= \int L (\chi) \nabla_{\theta} (\log p (\chi ; \theta)) p (\chi ; \theta) d \chi ,
\end{align*}
where we have used the fact that $\nabla_{\theta} \log p (\chi ; \theta)
= \inlinefrac{\nabla_{\theta} p (\chi ; \theta)}{p (\chi ; \theta)}$. With this
modified expression, we can estimate $\nabla_{\theta} \E_{\chi
\sim p (\chi ; \theta)} [L (\chi)]$ by sampling from $p (\chi ; \theta)$.

In our case, $\chi$ will be the trajectory of the stochastic hidden states of the \rnn,
and sampling from $p(\chi; \theta)$ will correspond to the following recursions:
\begin{align}
  h_t &= F (\bar{h}_{t-1}, x_t ; \theta_t) \notag \\
  \bar{h}_t &= h_t + \sigma u_t , \label{eqn:stochastic_dynamics}
\end{align}
with additive Gaussian noise $u_t \sim \mathcal{N}(0, I)$.
The stochastic hidden state $\bar{h}_t$ is effectively sampled from a state transition policy
$p (\bar{h}_t | \bar{h}_{t - 1}, \theta_t) \propto \exp \left( - \frac{1}{2 \sigma^2} \| \bar{h}_t - h_t \|^2 \right)$.

For each state $\bar{h}_t$ so visited,
we compute the score
$\nabla_{\theta} \log p (\bar{h}_{\leqslant t} ; \theta)$
of the trajectory $\bar{h}_{\leqslant t} = (\bar{h}_0, \bar{h}_1, \dots, \bar{h}_t)$
that brought us there, and
multiply it by an immediate loss $L_t$ so obtained.
Intuitively, higher
rewards ``reinforce'' directions in parameter space that bring them about.
We will assume $L_t$ is a differentiable function of $\bar{h}_t$.

By the chain rule of probability, the score
$\nabla_{\theta} \log p (\bar{h}_{\leqslant t} ; \theta)$
of the trajectory is simply the sum
 $\nabla_{\theta} \sum_{s=1}^t \log p (\bar{h}_s | \bar{h}_{s-1}, \theta_s )$,
which we can recursively maintain according to
\begin{align*}
  \bar{w}_t^{\top} & = \bar{w}_{t - 1}^{\top} +\nabla_{\theta} \log p
  (\bar{h}_t | \bar{h}_{t - 1}, \theta_t )
  = \bar{w}_{t - 1}^{\top} - \frac{1}{2 \sigma^2} \mathcal{J}_{h_t}^{\|
  \bar{h}_t - h_t \|^2} \dhtdwt{t}\\
  & = \bar{w}_{t - 1}^{\top} + \frac{1}{\sigma^2} (\bar{h}_t - h_t)^{\top}
  \dhtdwt{t}
  = \bar{w}_{t - 1}^{\top} + \frac{1}{\sigma} u_t^{\top} \dhtdwt{t}.
\end{align*}
Note that in the above computations, ``$\bar{h}_t$'' and ``$\bar{h}_{t-1}$'' are not the variables themselves but particular values. (This is a consequence of our adoption of the standard abuse of notation for random variables.) Thus they are treated as constants with respect to differentiation. The only quantity that depends on $\theta$ is $h_t$, which when we condition on the value of $\bar{h}_{t-1}$, only depends on $\theta$ via $\theta_t$.

This recursion is very similar to \uoro's recursion for $\wtilde_t^\top$, and it computes a similar type of sum: \begin{equation}
\label{eqn:barw_sum_form}
\bar{w}_t^{\top} = \frac{1}{\sigma} \sum_{s\leqslant t} u_s^{\top} \dhtdwt{s}.
\end{equation}

Once we have $\nabla_{\theta} \log p (\bar{h}_{\leqslant t} ; \theta)$, we need to
multiply it by the loss $L_t$ to obtain a \reinforce gradient estimate of
$\dltdw{t}$.
We can express the loss by its Taylor series around the point $u = 0$ where the noise is zero, as follows:
\begin{align*}
  L_t & = \left. L_t \right|_{u=0}
            + \Bigl( \sum_{s \leqslant t} \left. \jac{L_t}{u_s} \right|_{u=0} u_s \Bigr)
            + \frac{1}{2} \Bigl( \sum_{r \leqslant t} \sum_{s \leqslant t} u_r^{\top} \left.\mathcal{H}^{L_t}_{u_r, u_s} \right|_{u=0} u_s \Bigr)
            + \cdots \\
          & = \left. L_t \right|_{u=0} + \sigma \Bigl( \sum_{s \leqslant t} \left. \dltdhs{t}{s} \right|_{u=0} u_s \Bigr) + \bigO (\sigma^2) ,
\end{align*}
where $\mathcal{H}^{L_t}_{u_r, u_s}$ denotes the Hessian of $L_t$ with respect to $u_r$ and $u_s$.
The last step uses the fact that $\sigma u_s$ affects $L_t$ in exactly the same way that $h_s$ does,
so that $\jac{L_t}{u_s} = \sigma \dltdhs{t}{s}$ and $\mathcal{H}^{L_t}_{u_r, u_s} = \sigma^2 \mathcal{H}^{L_t}_{h_r, h_s}$.

Plugging the Taylor series for $L_t$ into the \reinforce gradient
estimate and using Equation \ref{eqn:barw_sum_form}, we get:
\begin{equation}
  L_t \bar{w}_t^{\top}
  = \left. L_t \right|_{u=0} \bar{w}_t^{\top}
        + \Bigl(\sum_{s \leqslant t} \left. \dltdhs{t}{s} \right|_{u=0} u_s \Bigr)
          \Bigl(\sum_{s \leqslant t} u_s^{\top} \dhtdwt{s} \Bigr) + \bigO (\sigma) .
          \notag \label{eqn:REINFORCE_UORO_link}
\end{equation}
Here we see the \uoro gradient estimator appear in the second term, but with
an important difference: the $\dhtdwt{s}$'s are evaluated in
the noisy system, whereas the $\left. \dltdhs{t}{s} \right|_{u=0}$ are evaluated with
zero noise. Thus this term doesn't estimate $\dltdw{t}$ for any value of $u$. However, the
equivalence becomes exact when we let the noise go to zero by taking the limit
$\sigma \rightarrow 0$.

To see this we first observe that letting $\sigma$ go to $0$ is equivalent to letting $u$ go to $0$ in the recursions for $h_s$ (Equation \ref{eqn:stochastic_dynamics}).
Furthermore, since $F$ is continuously differentiable, so is $h_s$ (w.r.t. all of its dependencies).
Therefore $\dhtdwt{s}$ is a continuous function of $u$, and it follows that
\begin{align*}
\lim_{\sigma \rightarrow 0} \dhtdwt{s} = \lim_{u \rightarrow 0} \dhtdwt{s} = \left. \dhtdwt{s} \right |_{u=0} .
\end{align*}
And therefore we have 
\begin{equation*}
\lim_{\sigma \rightarrow 0} \Bigl[
\Bigl( \sum_{s \leqslant t} \left. \dltdhs{t}{s} \right|_{u=0} u_s \Bigr)
\Bigl( \sum_{s \leqslant t} u_s^{\top} \dhtdwt{s} \Bigr)
+ \bigO (\sigma) \Bigr]
= \Bigl(\sum_{s \leqslant t} \left. \dltdhs{t}{s} \right|_{u=0} u_s \Bigr)
  \Bigl(\sum_{s \leqslant t} u_s^{\top} \left. \dhtdwt{s} \right|_{u=0} \Bigr) ,
\end{equation*}
which is identical to the standard \uoro estimate
$\dltdht{t} \htilde_t \wtilde_t^{\top}$
of
$\dltdw{t}$ (without any variance reduction).

Thus we can see that in the limit as $\sigma \rightarrow 0$, \reinforce becomes
equivalent to \uoro (sans variance reduction), except that it includes the additional term:
\begin{equation*}
\left. L_t \right|_{u=0} \bar{w}_t^{\top} = \frac{1}{\sigma} \left. L_t \right|_{u=0} \sum_{s\leqslant t} u_s^{\top} \dhtdwt{s} .
\end{equation*}
From the RHS expression we see that this term has mean zero, and thus the limiting behavior of \reinforce is to give an unbiased estimate of the gradient of the noise-free model. However, the variance of the additional term goes to infinity as $\sigma \rightarrow 0$.
For models where the noise is bounded away from zero this term represents the main source of variance for \reinforce estimators. It can however be addressed by subtracting an estimate of $\left. L_t \right|_{u=0}$ from $L_t$ before multiplying by the score function. This is known as a ``baseline'' in the \reinforce literature \citep{williams1992simple}.

The appearance of the \uoro estimator as part of the \reinforce estimator
suggests an additional opportunity for variance reduction in \reinforce.
If in Equation \ref{eqn:stochastic_dynamics} we had instead defined
\[
\bar{h}_t = h_t + \sigma Q_t u_t,
\]
that is, the noise added to $h_t$ has covariance $\sigma^2 Q_t^\top Q_t$,
then we would have found
\[
\bar{w}_t^{\top} = \frac{1}{\sigma} \sum_{s \leqslant t} u_s^{\top} Q^{-1} \dhtdwt{s}
\mbox{\quad and \quad }
\sum_{s \leqslant t} \left. \jac{L_t}{u_s} \right|_{u=0} = \sigma \sum_{s \leqslant t} \left. \jac{L_t}{h_s} \right|_{u=0} Q_s u_s.
\]
Putting these two together as in Equation \ref{eqn:REINFORCE_UORO_link} and
passing to the limit $\sigma \rightarrow 0$ as before, we get
\[
\lim_{\sigma \rightarrow 0} L_t \bar{w}_t^\top =
\left. L_t \right|_{u=0} \bar{w}_t^\top 
+ \Bigl(\sum_{s \leqslant t} \left. \dltdhs{t}{s} \right|_{u=0} Q_s u_s \Bigr)
  \Bigl(\sum_{s \leqslant t} u_s^{\top} Q_s^{-1} \left. \dhtdwt{s} \right|_{u=0} \Bigr),
\]
where now the second term is identical to \uoro \emph{with} the generalized
variance reduction described in Section \ref{sec:generalized_recursions}.
Thus the $Q_s$ matrices that enable variance reduction in \uoro
correspond directly to a choice of covariance on the exploration noise in \reinforce.

\section{Conclusions} \label{sec:conclusion}

We have contributed a thorough analysis of \uoro-style approximate differentiation algorithms
and their variance behavior.
The theory takes a holistic view of the algorithm as part of an optimization process,
where the sequence of mutually dependent gradient estimates
$\dltdw{t} \approx \dltdht{t} \htilde_t \wtilde_t^\top$
produced by \uoro are accumulated as per gradient descent.
Our analysis considers the variance of this total gradient estimate.
This is in contrast to \uoro's variance reduction scheme (\gir)
which minimizes the variance of individual Jacobian estimates
$\dhtdw{t} \approx \htilde_t \wtilde_t^\top$,
without accounting for the way in which they are used.
We have developed a generalization of \gir,
and suggested avenues toward a practical implementation.
Empirical evaluation confirms our theoretical claims.

Furthermore we have described an variation on \uoro that avoids ``spatial'' projection,
greatly reducing the variance at the cost of increased computational complexity.
Finally, we have drawn a deep connection between \uoro and \reinforce when the latter
is used to train an \rnn with perturbed hidden states.

\acks{%
The authors thank Max Jaderberg, David Sussillo, David Duvenaud and Aaron Courville for helpful discussion,
and Chris Maddison and Grzegorz Swirszcz for reviewing drafts of this paper.
This research was enabled by computational resources courtesy of Compute Canada.
}

\newpage
\appendix

\section{Supporting Results for Variance Computations} \label{sec:secondmoment}

In this section we prove several technical results supporting our variance computations in the main text.

\begin{definition}[Standard random vector]
A standard random vector is any real vector $u$ whose elements $u_i$ are drawn iid
from a distribution that is symmetric around zero and has unit variance.
\end{definition}

Standard random vectors $u$ satisfy $\E [u] = 0$ and $\E [u u^\top] = I$,
which is required for our algorithms to be unbiased.
Moreover, by symmetry the odd moments of their elements $u_i$ are zero.
The results below will involve the ``excess kurtosis'' $\E [u_1^4] - 3$
of the distribution of the elements of $u$.
The standard normal distribution $\mathcal{N}(0, 1)$ has excess kurtosis 0,
whereas the uniform distribution on signs $\mathcal{U}\{-1,+1\}$ has excess kurtosis -2.

\begin{proposition} \label{prop:secondmoment}
Suppose $A, B, C, D$ are constant matrices, and $u$ is a standard random vector
with excess kurtosis $\kappa$.
Then we have
\begin{align*}
    \E [ A u u^{\top} B C u u^{\top} D ] =
    \trace (B C) A D
    + 2 A B C D
    + \kappa A ((B C) \odot I) D.
\end{align*}
\end{proposition}

\begin{proof}

By linearity of expectation,
\begin{equation*}
  \E [ A u u^{\top} B C u u^{\top} D ] =
  \sum_{j k l m} \E [ u_j u_k u_l u_m ] A_{i j} (B C)_{k l} D_{m n}.
\end{equation*}
In order to evaluate the expectation $\E [ u_j u_k u_l u_m ]$, we make use
of the fact that $u_i$ and $u_j$ are independent unless $i = j$.
This allows us to express the product inside the expectation
as a product of powers $u_i^{p(i)}$, with the power $p(i)$ equal to the multiplicity of $i$ in $(j, k, l, m)$.
By independence, the expectation of this product then factors into a product $\prod_i \E [ u_i^{p(i)} ] = \prod_i \mu_{p(i)}$ of moments $\mu_p \triangleq \E [ u_1^p ]$ of the elements $u_i$.
Moreover, since by symmetry the odd moments of $u_i$ are zero, we need only consider
cases in which all indices have even multiplicity.
Thus we get
\begin{equation*}
  \sum_{j k l m} \E [ u_j u_k u_l u_m ] A_{i j} (B C)_{k l} D_{m n}
  = \sum_{j k l m} \begin{dcases*}
    \mu_4   A_{i j} (B C)_{j j} D_{j n} & if $j = k = l = m$ \\
    \mu_2^2 A_{i j} (B C)_{j l} D_{l n} & if $j = k \neq l = m$ \\
    \mu_2^2 A_{i j} (B C)_{k j} D_{k n} & if $j = l \neq k = m$ \\
    \mu_2^2 A_{i j} (B C)_{k k} D_{j n} & if $j = m \neq k = l$ \\
    0 & else
  \end{dcases*}.
\end{equation*}

Casting this back into matrix form, we have
\begin{align*}
  \E [ A u u^{\top} B C u u^{\top} D ] =~ &
    \mu_2^2 \trace (B C) A D
    + 2 \mu_2^2 A B C D
    + (\mu_4 - 3 \mu_2^2) A ((B C) \odot I) D
    \\ =~ &
    \trace (B C) A D
    + 2 A B C D
    + \kappa A ((B C) \odot I) D,
\end{align*}
where $\mu_2^2 = 1$ follows from the fact that $u$ is a standard random vector,
and $\mu_4 - 3 \mu_2^2 = \mu_4 - 3 = \kappa$
is its excess kurtosis.

\end{proof}

\begin{corollary}
\label{cor:variance}
Suppose $x$ and $y$ are constant vectors, $V$ and $W$ are constant matrices,
and $u$ is a standard random vector with excess kurtosis $\kappa$.
Then
\begin{equation*}
    \Cov [ x^\top u u^\top V, y^\top u u^\top W ] =
    (x^\top y) V^{\top} W + V^{\top} x y^{\top} W
    + \kappa V^\top ((x y^\top) \odot I) W
\end{equation*}
and
\begin{align*}
\trace \bigl( \Cov [ x^\top u u^\top V, y^\top u u^\top W ] \bigr)
=\;& (x^\top y) \trace (V^{\top} W) + 2 \trace (V^{\top} x y^{\top} W)
\\    &+ \kappa \trace \bigl( V^\top ((x y^\top) \odot I) W \bigr)
.
\end{align*}

\end{corollary}
\begin{proof}
$x^\top u u^\top V$ and $y^\top u u^\top W$ are row vectors and so their covariance is given by
\begin{align*}
    \Cov [x^\top u u^\top V, y^\top u u^\top W]
    &=
    \E [(x^\top u u^\top V)^\top (y^\top u u^\top W)] - \E [x^\top u u^\top V]^\top \E [y^\top u u^\top W] \\
    &= \E [V^\top u u^\top x y^\top u u^\top W] - \E [V^\top u u^\top x] \E [y^\top u u^\top W] .
\end{align*}

By Proposition \ref{prop:secondmoment},
\begin{align*}
\E [V^\top u u^\top x y^\top u u^\top W]
=\;& \trace (x y^{\top}) V^{\top} W + 2 V^{\top} x y^{\top} W + \kappa V^\top ((x y^\top) \odot I) W
\\=\;& (x^\top y) V^{\top} W + 2 V^{\top} x y^{\top} W  + \kappa V^\top ((x y^\top) \odot I) W.
\end{align*}
And by linearity of expectation we have
$\E [y^{\top} u u^{\top} W] = y^{\top} \E [u u^{\top}] W = y^{\top} W$
and similarly
$\E [V^{\top} u u^{\top} x] = V^{\top} x$,
so that $\E [V^\top u u^\top x] \E [y^\top u u^\top W] = V^\top x y^\top W$.
Combining these equations yields
\begin{align*}
\Cov [x^\top u u^\top V, y^\top u u^\top W]
=\;& (x^\top y) V^{\top} W + V^{\top} x y^{\top} W + \kappa V^\top ((x y^\top) \odot I) W.
\end{align*}
The formula for $\trace \bigl(\Cov [x^\top u u^\top V, y^\top u u^\top W] \bigr)$ follows immediately.
\end{proof}

\section{Variance of a Single Jacobian Estimate} \label{sec:gir_limitations_variance}

Section \ref{sec:gir_limitations} discusses the following expression for the \uoro Jacobian estimate at time $t$ in terms of the overall coefficients $\alpha_r^{(t)}$:
\begin{equation} \label{eqn:single_estimate_expanded_recap}
\dhtdw{t} \approx
\htilde_t \wtilde_t^\top
=
\Bigl(\sum_{s \leqslant t} \alpha_s^{(t)} \dhtdzs{t}{s} u_s \Bigr)
\Bigl(\sum_{r \leqslant t} \frac{1}{\alpha_r^{(t)}} u_r^\top \dztdwt{r} \Bigr)
= \sum_{r \leqslant t} \sum_{s \leqslant t}
\frac{\alpha_s^{(t)}}{\alpha_r^{(t)}}
\tau_s \tau_r
   \dhtdhs{t}{s} \dhtdzt{s}
   \nu_s \nu_r^\top
   \dztdwt{r}.
\end{equation}
This section concerns the variance of this estimate and the coefficients $\alpha_s^{(t)}$ that minimize it.
We will omit the superscript on $\alpha_s^{(t)}$ to avoid notational clutter.

Defining $R^\top = \begin{pmatrix} \dhtdzs{t}{1} \nu_1 & \cdots & \dhtdzs{t}{t} \nu_t \end{pmatrix}$,
$\tilde{J}^\top = \begin{pmatrix} \tilde{J}_1 & \cdots & \tilde{J}_t \end{pmatrix}$ for $\tilde{J}_s^\top = \nu_s^\top \dztdwt{s}$,
and the diagonal matrix $A = \diag(\alpha)$,
we can write Equation \ref{eqn:single_estimate_expanded_recap} as
\begin{align*}
\htilde_t \wtilde_t^\top
=
R^\top A \tau \tau^\top A^{-1} \tilde{J}.
\end{align*}
Its variance with respect to the temporal noise $\tau$ is given by
\begin{align*}
&
\E_\tau \bigl[ \| R^\top A \tau \tau^\top A^{-1} \tilde{J} \|^2_F \bigr] -
\bigl\| \E_\tau [ R^\top A \tau \tau^\top A^{-1} \tilde{J} ] \bigr\|^2_F
\\ =\;&
\trace \bigl( \E_\tau [ R^\top A \tau \tau^\top A^{-1} \tilde{J} \tilde{J}^\top A^{-1} \tau \tau^\top A R ] \bigr) -
\| R^\top \tilde{J} \|^2_F
\\ =\;&
\trace ( R R^\top A^2 ) \trace ( \tilde{J} \tilde{J}^\top A^{-2} ) +
\| R^\top \tilde{J} \|^2_F
-2 \trace \bigl( R^\top ((\tilde{J} \tilde{J}^\top) \odot I) R \bigr)
,
\end{align*}
where in the last step we have made use of Proposition \ref{prop:secondmoment} (with $\kappa = -2$) to evaluate the second moment.
The part that depends on $\alpha$ is
\begin{align*}
 \trace(R R^\top A^2 ) \trace( \tilde{J} \tilde{J}^\top A^{-2} )
= \sum_{q \leqslant t} \sum_{r \leqslant t} \frac{\alpha_r^2}{\alpha_q^2} \| \dhtdzs{t}{r} \nu_r \|^2 \| \nu_q^\top \dztdwt{q} \|^2
= \sum_{q \leqslant t} \sum_{r \leqslant t} \frac{\alpha_r^2}{\alpha_q^2} C_{q r}
\end{align*}
where $C_{q r} \triangleq m_q n_r \triangleq \| \dhtdzs{t}{r} \nu_r \|^2 \| \nu_q^\top \dztdwt{q} \|^2$.
From the analysis in Appendix \ref{sec:alphas_optimization_algorithm} we know that this is minimal iff
\begin{align*}
e_k^\top A^{2} C A^{-2} \ones = e_k^\top A^{-2} C^\top A^2 \ones,
\end{align*}
where $e_k$ is the $k$th column of the identity matrix and $\ones$ is the vector of ones.
Using the rank-one structure of $C$, we have
\begin{align*}
\alpha_k^2 m_k n^\top A^{-2} \ones = \alpha_k^{-2} n_k m^\top A^2 \ones,
\end{align*}
which leads to the solution
\begin{align*}
\alpha_k^4 = \frac{n_k}{m_k} \frac{m^\top A^2 \ones}{n^\top A^{-2} \ones}
\propto \frac{n_k}{m_k} = \frac{\| \nu_k^\top \dztdwt{k} \|^2}{\| \dhtdzs{t}{k} \nu_k \|^2}.
\end{align*}

\section[Optimizing alphas given Q0]{Optimizing $\alpha$ given $Q_0$} \label{sec:alphas_optimization_algorithm}

Section \ref{sec:optimizing_alpha} introduced the following optimization problem (Equation \ref{eqn:alphas_optimization_problem}):
\begin{align*}
\alpha^\star =
    \operatornamewithlimits{argmin}_{\alpha > 0}
    \sum_{r \leqslant T} \sum_{q \leqslant T} \frac{\alpha_r^2}{\alpha_q^2} C_{q r}
\end{align*}
Here we analyze this problem in terms of a logarithmic parameterization $\alpha_i^2 = \exp(\zeta_i)$.
The coefficients $\exp(\zeta_i)$ give rise to diagonal column- and row-scaling matrices $Z, Z^{-1}$ with $Z_{i j} = \delta_{i j} \exp(\zeta_i)$.
These matrices act on $C$ to produce a modified matrix $\bar{C} = Z^{-1} C Z$,
of which $V(Q)$ is the elementwise sum:
\begin{equation*}
V(Q) = \sum_{r \leqslant T} \sum_{q \leqslant T} \exp(-\zeta_q) C_{q r} \exp(\zeta_r) = \ones^\top Z^{-1} C Z \ones = \ones^\top \bar{C} \ones.
\end{equation*}
By $\ones$ we denote the vector of ones.

We will make use of the matrix differential
\begin{equation*}
\frac{\dd \bar{C}}{\dd \zeta_k} = \onehot{k} \onehot{k}^\top \bar{C} - \bar{C} \onehot{k} \onehot{k}^\top
\end{equation*}
which measures the first-order change in $\bar{C}$ with respect to $\zeta_k$.
Here $\onehot{k}$ is the $k$th column of the identity matrix.
From this we get the derivative of $V(Q)$ with respect to $\zeta_k$:
\begin{equation*}
\frac{\dd V(Q)}{\dd \zeta_k} = \ones^\top \frac{\dd \bar{C}}{\dd \zeta_k} \ones = \onehot{k}^\top \bar{C} \ones - \ones^\top \bar{C} \onehot{k}.
\end{equation*}
The stationary points of $V(Q)$ satisfy
$\bar{C} \ones = \bar{C}^\top \ones$,
i.e. the modified matrix $\bar{C}$ has equal column and row sums.

Using the matrix differential $\frac{\dd \bar{C}}{\dd \zeta_k}$ twice, we find the
elements of the Hessian $H$:
\begin{align*}
H_{i j} &= \frac{\dd}{\dd \zeta_i} \frac{\dd}{\dd \zeta_j} \ones^\top \bar{C} \ones
= \frac{\dd}{\dd \zeta_i} \ones^\top \bigl( \onehot{j} \onehot{j}^\top \bar{C} - \bar{C} \onehot{j} \onehot{j}^\top \bigl) \ones
\\ &= \ones^\top \bigl(
  \onehot{j} \onehot{j}^\top \onehot{i} \onehot{i}^\top \bar{C} -
  \onehot{j} \onehot{j}^\top \bar{C} \onehot{i} \onehot{i}^\top -
  \onehot{i} \onehot{i}^\top \bar{C} \onehot{j} \onehot{j}^\top +
  \bar{C} \onehot{i} \onehot{i}^\top \onehot{j} \onehot{j}^\top
\bigr) \ones
\\ &= \delta_{i j} \onehot{i}^\top \bar{C} \ones - \onehot{j}^\top \bar{C} \onehot{i} - \onehot{i}^\top \bar{C} \onehot{j} + \delta_{i j} \ones^\top \bar{C} \onehot{i}
\end{align*}
which in matrix form is
\begin{align*}
H = \diag ( \bar{C} \ones ) - \bar{C} + \diag ( \bar{C}^\top \ones ) - \bar{C}^\top
  = \diag ( (\bar{C} + \bar{C}^\top ) \ones ) - ( \bar{C} + \bar{C}^\top ).
\end{align*}
It is easily shown that the Hessian is positive semidefinite everywhere and hence $V(Q)$ is convex in $\zeta$ for all real vectors $v$:
\begin{align*}
v^\top H v
&= v^\top \diag ( ( \bar{C} + \bar{C}^\top ) \ones ) v - v^\top ( \bar{C} + \bar{C}^\top ) v
= \sum_{i j} (v_i^2 - v_i v_j) ( \bar{C} + \bar{C}^\top )_{i j}
\\ &= \frac{1}{2} \sum_{i j} (2 v_i^2 - 2 v_i v_j) ( \bar{C} + \bar{C}^\top )_{i j}
= \frac{1}{2} \sum_{i j} (v_i^2 + v_j^2 - 2 v_i v_j) ( \bar{C} + \bar{C}^\top )_{i j}
\\ &= \frac{1}{2} \sum_{i j} (v_i - v_j)^2 ( \bar{C} + \bar{C}^\top )_{i j} .
\end{align*}
As $(v_i - v_j)^2 \geqslant 0$ and $(\bar{C} + \bar{C}^\top)_{i j} \geqslant 0$ due to positivity of the entries of $C$, each term in the sum is nonnegative and therefore the whole sum is nonnegative. This implies $H$ is positive semidefinite and hence $V(Q)$ is convex.

Given that $V(Q)$ is smooth and convex, its stationary points are global minimizers.
In our experiments we solve for the stationary points by Newton's method,
according to the update
\[
\zeta \leftarrow \zeta - \eta (H + \lambda I)^{-1} (\bar{C} - \bar{C}^\top) \ones
\]
where $\eta$ is a learning rate and $\lambda$ is a damping factor on $H$, which is necessary because one of its eigenvalues is zero.
In our experiments, we use $\eta = 1$ and $\lambda = 10^{-8}$.

\section[Online optimization of alpha coefficients]{Online optimization of $\alpha$ coefficients} \label{sec:greedy_alphas_example}

This Appendix demonstrates how the incremental formulation of the optimization with respect to $\alpha$
from Section \ref{sec:optimizing_alpha} (Equation \ref{eqn:greedy_alphas_optimization_problem})
may be used to derive practical values for the $\gamma, \beta$ coefficients.
Recall that the optimization problem is defined in terms of a matrix $C^{(s)}$
that stands in for the unknown $C$.
We will work with a naive choice that assumes future gradients and Jacobians are zero:
\begin{align*}
C^{(s)}_{q r}
   &\triangleq
   \bigl\| \sum_{t = q}^s b_r^{(t) \top} Q_0 \bigr\|^2 \bigl\| Q_0^{-1} J_q \bigr\|^2_F.
\end{align*}
Note that $\| \sum_{t = q}^s b_r^{(t) \top} Q_0 \|^2$ is zero unless $q \leqslant s$ and $r \leqslant s$, and thus $C^{(s)}_{q r} = \indicator{q \leqslant s} \indicator{r \leqslant s} C^{(s)}_{q r}$.
Using this property, we can rewrite the problem (Equation \ref{eqn:greedy_alphas_optimization_problem}) as
\[
\beta_s^\star, \gamma_s^\star =
    \operatornamewithlimits{argmin}_{\beta_s, \gamma_s}
\sum_{r \leqslant s} \sum_{q \leqslant s}
\frac{\alpha_r^2}{\alpha_q^2}
C^{(s)}_{q r}.
\]
Expanding
\[
\frac{\alpha_r^2}{\alpha_q^2}
= \frac{\beta_r^2 \gamma_{r+1}^2 \dots \gamma_T^2}
       {\beta_q^2 \gamma_{q+1}^2 \dots \gamma_T^2}
= \frac{\beta_r^2 \gamma_{r+1}^2 \dots \gamma_q^2}
       {\beta_q^2 \gamma_{q+1}^2 \dots \gamma_r^2}
\]
reveals that only terms with either $q = s$ or $r = s$ depend on $\beta_s$ and/or $\gamma_s$,
and thus
\[
\beta_s^\star, \gamma_s^\star =
    \operatornamewithlimits{argmin}_{\beta_s, \gamma_s}
  \frac{\beta_s^2}{\gamma_s^2} \sum_{q < s} \beta_q^{-2} \gamma_{q+1}^{-2} \dots \gamma_{s-1}^{-2} C^{(s)}_{q s}
+ \frac{\gamma_s^2}{\beta_s^2} \sum_{r < s} \beta_r^2 \gamma_{r+1}^2 \dots \gamma_{s-1}^2 C^{(s)}_{s r}.
\]
Note that $\beta_s$ and $\gamma_s$ appear through the single degree of freedom $\inlinefrac{\beta_s^2}{\gamma_s^2}$,
and by differentiation we find the stationary points
\[
\frac{\gamma_s^4}{\beta_s^4}
=
\frac{\sum_{q<s} \beta_q^{-2} \gamma_{q+1}^{-2} \dots \gamma_{s-1}^{-2} C^{(s)}_{q s}}
     {\sum_{r<s} \beta_r^2 \gamma_{r+1}^2 \dots \gamma_{s-1}^2 C^{(s)}_{s r}}
\]
From our definition of $C^{(s)}$ we have that
\[
C^{(s)}_{q s}
=  \| b_s^{(s) \top} Q_0 \|^2 \| Q_0^{-1} J_q \|^2_F
\text{\quad and \quad}
C^{(s)}_{s r}
=  \| b_r^{(s) \top} Q_0 \|^2 \| Q_0^{-1} J_s \|^2_F,
\]
which leads to the natural solution
\[
\beta_s^4 =
\frac{\| Q_0^{-1} J_s \|^2_F}{\| b_s^{(s) \top} Q_0 \|^2}
\text{\quad and \quad}
\gamma_s^4 = \frac{
\sum_{q<s} \| \beta_q^{-1} \gamma_{q+1}^{-1} \dots \gamma_{s-1}^{-1} Q_0^{-1} J_q \|^2_F
}{
\sum_{r<s} \| \beta_r \gamma_{r+1} \dots \gamma_{s-1} b_r^{(s) \top} Q_0 \|^2
}.
\]
It can be shown that the above solution can be expressed in terms of ratios of expectations of familiar quantities:
\[
\beta_s^4 =
\frac{\E_u \| u_s^\top Q_0^{-1} \dztdwt{s} \|^2_F}
     {\E_u \| \dltdht{s} \dhtdzt{s} Q_0 u_s \|^2 \vphantom{\tilde{h}}}
\text{\quad and \quad}
\gamma_s^4 =
\frac{\E_u \| \wtilde_{s-1} \|^2_F}
     {\E_u \| \dltdht{s} \dhtdhs{s}{s-1} \htilde_{s-1} \|^2 \vphantom{\tilde{h}}}
\]
These coefficients are closely related to those of \gir as derived in Section \ref{sec:uoro_gir}.
In fact, had we defined
\begin{align*}
C^{(s)}_{q r}
   &\triangleq
   \| \dhtdzs{s}{r} Q_0 \|^2 \| Q_0^{-1} J_q \|^2_F,
\end{align*}
the projection onto $\dltdzt{s}$ would disappear from the coefficients,
making the similarity even more striking.
However, this choice is not consistent with our objective of minimizing the variance $V(Q)$ of the total gradient estimate.

Note that we were able to solve for the coefficients in closed form thanks to the property
$C^{(s)}_{q r} = \indicator{q \leqslant s} \indicator{r \leqslant s} C^{(s)}_{q r}$.
In general, solving Equation \ref{eqn:greedy_alphas_optimization_problem}
involves joint optimization of $\beta_{\geqslant s}, \gamma_{\geqslant s}$,
which requires a numerical approach similar to the one described in Appendix \ref{sec:alphas_optimization_algorithm}.
Moreover, $\beta_s$ and $\gamma_s$ will in general be independent parameters.

\section{Minimization of the Product of Traces} \label{sec:minimization}

Minimizing the total variance of our estimators involves minimizing
a product of traces by choice of a noise-shaping matrix. Here we
characterize the optimal choice of such a matrix in a general setting. 

\begin{definition}
Define $c (A) = \trace (XA) \trace (YA^{- 1})$
for PD matrices $X$ and $Y$.
\end{definition}

The goal of this section will be to prove the following theorem.
\begin{theorem}
\label{theorem:minimizer}
A PD matrix $A$ is a global minimizer of $c (A)$ over the
set of PD matrices if and only if
\[ XA = \gamma A^{- 1} Y \]
for some scalar $\gamma > 0$.
\end{theorem}

Note that a similar result to Theorem \ref{theorem:minimizer} one was used implicitly by \citet{ollivier2015training}, but wasn't given rigorous justification.  It is relatively easy to characterize the critical points of $c(A)$, but proving that any critical point is a global minimizer is much more involved. It would be tempting to use convexity to prove such a result but unfortunately $c(A)$ is not convex in general.

We begin by stating and proving some basic technical claims.

\begin{claim}
\label{claim:similar}
Let $U$ be a matrix and $V$ be a PD matrix with $V =
CC^{\top}$ for some $C$. Then the eigenvalues of $UV$ are the same as the
eigenvalues of $C^{\top} UC$.
\end{claim}
\begin{proof}
Observe that
\[ C^{\top}  (UV) C^{- \top} = C^{\top}  (UCC^{\top}) C^{- \top} = C^{\top}
   UC. \]
Thus $UV$ is similar to the matrix $C^{\top} UC$ and so has the same
eigenvalues.
\end{proof}
\begin{corollary}
If $U$ and $V$ are PD matrices then they have all
positive eigenvalues and $\trace (UV) > 0$.
\end{corollary}

\begin{claim}
\label{claim:critical}
$A$ is a critical point of $c$ if and only if $XA =
\gamma A^{- 1} Y$ for some $\gamma > 0$.
\end{claim}
\begin{proof}
Differentiating $c (A)$ with respect to $A$, we find
\begin{align*}
  \frac{d c (A)}{d A} & = \frac{d}{d A} \trace (XA) \trace (YA^{-
  1})\\
  & = \trace (YA^{- 1}) \frac{d}{d A} \trace (XA) + \trace (XA)
  \frac{d}{d A} \trace (YA^{- 1})\\
  & = \trace (YA^{- 1}) X^{\top} - \trace (XA) A^{- \top} Y^{\top}
  A^{- \top} .
\end{align*}
Setting this to zero and rearranging terms gives
\[ XA = \frac{\trace (XA)}{\trace (A^{- 1} Y)} A^{- 1} Y . \]
Because $A$, $A^{- 1}$, $X$, and $Y$ are all PD matrices, and the trace of a
product of PD matrices is positive by the previous claim, the result follows.
\end{proof}

\begin{claim}
\label{claim:minimum}
$c (A) = \trace ( (XY)^{\superfrac{1}{2}} )^2$
for critical points $A$, where $(XY)^{\superfrac{1}{2}}$ is the (unique) positive
square root of $XY$.
\end{claim}
\begin{proof}
Let $A$ be a critical point. By Claim \ref{claim:critical} we have
$XA = \gamma A^{- 1} Y$ for some $\gamma > 0$.
This implies that $X = \gamma A^{- 1} YA^{- 1}$ and thus $XY = \gamma A^{- 1}
YA^{- 1} Y = \gamma (A^{- 1} Y)^2$.
Because $A^{- 1}$ and $Y$ are PD matrices we have by Claim \ref{claim:similar} that $A^{-
1} Y$ has all positive eigenvalues. Thus $A^{- 1} Y = \gamma^{-\superfrac{1}{2}}
(XY)^{\superfrac{1}{2}}$ where $(XY)^{\superfrac{1}{2}}$ is the (unique) positive square
root of $XY$.
We also have that $Y = \gamma^{-1} AXA$ which implies that $XY =
 \gamma^{-1}  (XA)^2$, and so by a similar argument to the one above we
have that $XA = \gamma^{\superfrac{1}{2}} (XY)^{\superfrac{1}{2}}$.
Thus
\[
c (A)
= \trace (XA) \trace (YA^{- 1})
= \trace ( \gamma^{\superfrac{1}{2}} (XY)^{\superfrac{1}{2}} )
  \trace ( \gamma^{-\superfrac{1}{2}}  (XY)^{\superfrac{1}{2}} )
 = \trace ( (XY)^{\superfrac{1}{2}} )^2.
 \]
\end{proof}

\begin{observation}
Note that $\trace ( (XY)^{\superfrac{1}{2}})^2$ depends only on the eigenvalues of $XY$ and so for any other matrix
$V$ with the same eigenvalues $\trace ( V^{\superfrac{1}{2}} )^2$
would also give us the value of $c (A)$ at critical points. By Claim \ref{claim:similar}
such choices include $X^{\superfrac{1}{2}} YX^{\superfrac{1}{2}}$ and $Y^{\superfrac{1}{2}} XY^{\superfrac{1}{2}}$.
\end{observation}

\begin{definition}
Let $\lambda_{\min} (V)$ denote the minimum eigenvalue of $V$.
\end{definition}

Note that we may restrict our analysis of $c$ to the following domain:
\[ \mathcal{A}_1 = \left\{ A : A \text{ is PD and } \lambda_{\min} (A) = 1
   \right\} . \]
This is because
$c (\alpha A) = \trace (X (\alpha A)) \trace (Y (\alpha A)^{- 1})
= \alpha \trace (XA)  \alpha^{-1} \trace (YA^{- 1}) = c(A)$,
and so we can always replace $A$ with $A / \lambda_{\min} (A)
\in \mathcal{A}_1$ without changing the objective function value. (Note that
$\lambda_{\min} (A) > 0$ since $A$ is PD, so the new matrix
$A / \lambda_{\min} (A)$ remains PD.)

The remainder of this section will be devoted to showing that $c$, when
restricted to $\mathcal{A}_1$, attains its minimum on that set. Combining this
with the fact that $c$ is continuously differentiable on the (larger) set of
PD matrices, we will thus have that some critical point is a global minimizer
of $c$ on the set of all PD matrices.

And since all critical points have the same objective function value by
Claim \ref{claim:minimum}, it will follow that all critical points are global minimizers.
And so by Claim \ref{claim:critical}, we will have that $A$ is a global minimizer if and only if
\[ XA = \gamma A^{- 1} Y \]
for some $\gamma > 0$.

\begin{claim}
\label{claim:eigendecomposition}
Let $A$ be some PD matrix with eigendecomposition given
by $A = U \diag (d) U^{\top}$. Then we have
\[ c (A) = \Bigl( \sum_i d_i  (u_i^{\top} Xu_i) \Bigr)  \Bigl( \sum_i d_i^{-1} (u_i^{\top} Yu_i) \Bigr), \]
where $u_i$ is the $i$-th column of $U$ (i.e. the $i$-th eigenvector of $A$)
and $d_i$ is the $i$-th entry of the vector $d$.
\end{claim}
\begin{proof}
Observe that
\[ \trace (XA) = \trace (XU \diag (d) U^{\top}) = \trace (U^{\top} XU \diag (d)) . \]
Noting that the $i$-th diagonal element of $U^{\top} XU$ is $u_i^{\top} Xu_i$,
so that the $i$-th diagonal element of $U^{\top} XU \diag (d)$ is $d_i 
(u_i^{\top} Xu_i)$, it follows that
\[ \trace (XA) = \sum_i d_i  (u_i^{\top} Xu_i) . \]
Observing that $A^{- 1} = U \diag (d)^{- 1} U^{\top}$ so that $A^{- 1} U
= U \diag (d)^{- 1}$, and that the $i$-th diagonal element of
$\diag (d)^{- 1}$ is just $d_i^{- 1}$ we can apply a similar argument to
the above to show that
\[ \trace (YA^{- 1}) = \sum_i d_i^{-1} (u_i^{\top} Yu_i) . \]
Combining these equation equations establishes the claim.
\end{proof}

\begin{claim}
\label{claim:closed}
$\mathcal{A}_1$ is a closed set.
\end{claim}
\begin{proof}
Note that $\lambda_{\min} (A)$ is a continuous function of $A$,
and $\{ 1 \}$ is a closed set. Morever the set of PSD matrices is a closed
set. We therefore have that the intersection of the preimage of
$\lambda_{\min} (A)$ on $\{ 1 \}$ and the set of PSD matrices is a closed set
(i.e. the set $\{ A : A \text{ is PSD and } \lambda_{\min} (A) = 1
\}$) is closed. But this set is precisely $\mathcal{A}_1$ since any PSD
matrix with $\lambda_{\min} (A) = 1$ is also clearly PD.
\end{proof}

\begin{definition}
A function $f : S \rightarrow \mathbbm{R}$ defined on
a set $S \subset \mathbbm{R}^n$ is called ``coercive'' if we have
\[ f (v) \rightarrow \infty \text{\qquad as\qquad} \| v \| \rightarrow \infty
   . \]
\end{definition}

The following is a standard result in finite dimensional analysis~\cite[Chapter 6]{heath2018scientific}:

\begin{theorem}
\label{theorem:coercive}
If $f : S \rightarrow \mathbbm{R}$ is coercive and
continuous and $S \subset \mathbbm{R}^n$ is a closed set then $f$ obtains in
minimum on $S$.
\end{theorem}

Note that the theorem applies equally to the space of finite dimensional
real-valued matrices where the norm is any valid matrix norm, including the
standard spectral norm (which our notation will assume).

\begin{claim}
\label{claim:coercive}
$c (A)$ is coercive on the set $\mathcal{A}_1$.
\end{claim}
\begin{proof}
Let $A \in \mathcal{A}_1$ with eigendecomposition given by
$A = U \diag (d) U^{\top}$. By Claim \ref{claim:eigendecomposition} we have that
\[ c (A) = \Bigl( \sum_i d_i  (u_i^{\top} Xu_i) \Bigr)  \Bigl( \sum_i d_i^{-1} (u_i^{\top} Yu_i) \Bigr) . \]
Since $A \in \mathcal{A}_1$ we can assume without loss of generality that $d_1 = 1$. We can also
assume without loss of generality that $d_2$ is the largest eigenvalue of $A$ so that $d_2 = \| A
\|$.

Because $A$, $X$ and $Y$ are all PD we have that $u_i^{\top} Xu_i \geqslant
\lambda_{\min} (X) > 0$ and $u_i^{\top} Yu_i \geqslant \lambda_{\min} (Y) > 0$
for each $i$. And thus
\begin{align*}
  c (A) & \geqslant \lambda_{\min} (X) \lambda_{\min} (Y)
  \Bigl( 1 + \| A \| + \sum_{i > 2} d_i \Bigr) \Bigl( 1 + \| A \|^{-1} + \sum_{i > 2} d_i^{-1} \Bigr)\\
  & \geqslant \lambda_{\min} (X) \lambda_{\min} (Y)  (1 + \| A \|)  (1 + \| A \|^{-1}) .
\end{align*}
Clearly this goes to infinity as $\| A \|$ does, which establishes the claim.
\end{proof}

\begin{claim}
\label{claim:minimumona1}
$c (A)$ attains its minimum on the set $\mathcal{A}_1$.
\end{claim}
\begin{proof}
This follows directly from Claims \ref{claim:closed} and \ref{claim:coercive} and
Theorem \ref{theorem:coercive}.
\end{proof}

\section[Estimating B online]{Estimating $B$ online} \label{sec:estimating_B}

The optimal $Q_0$ derived in Section \ref{sec:optimizing_q0} depends on the matrix $B$ (Equation \ref{eqn:b_matrix})
which is unknown due to its dependence on future coefficients and gradients.
To obtain a practical algorithm, we must approximate it online.
We will consider $B$ to be the final element in a sequence $B^{(1)} \dots B^{(T)}$ of matrices that accumulate information observed so far:
\[
B^{(k)} = \sum_{s \leqslant k} \sum_{t \leqslant k}
\Bigl( \sum_{q=1}^{\min(s, t)} \alpha_q^{-2} \| a_q \|^2 \Bigr)
\Bigl( \sum_{r=1}^{\min(s, t)} \alpha_r^{2} b^{(s)}_r b^{(t)\top}_r \Bigr).
\]
The remainder of this section develops an online algorithm that produces an unbiased estimate of
$B^{(s)}$ at each step $s$.
Although this will not yield an unbiased estimate of $B^{(T)}$ until the final time step $T$,
to the extent that $B^{(s)}$ is stationary we may use its intermediate estimates $B^{(s)}$ for $s < T$ as approximations to $B^{(T)}$.

First, we factorize the sums over $q$ and $r$ using the now-familiar random projections onto independent temporal noise vectors $\sigma$ and $\tau$:
\[
\sum_{q=1}^{\min(s, t)} \alpha_q^{-2} \| a_q \|^2
=
\E_\sigma \Bigl[
\Bigl( \sum_{q \leqslant s} \sigma_q \alpha_q^{-1} \| a_q \| \Bigr)
\Bigl( \sum_{q \leqslant t} \sigma_q \alpha_q^{-1} \| a_q \| \Bigr)
\Bigr]
\]
and
\[
\sum_{r=1}^{\min(s, t)} \alpha_r^{2} b^{(s)}_r b^{(t)\top}_r
=
\E_\tau \Bigl[
\Bigl( \sum_{r \leqslant s} \tau_r \alpha_r b^{(s)}_r \Bigr)
\Bigl( \sum_{r \leqslant t} \tau_r \alpha_r b^{(t)}_r \Bigr)^\top
\Bigr].
\]
By doing so we have broken up the dependency on $\min(s, t)$ into
separate factors.
Defining $\tilde{a}_s = \sum_{q \leqslant s} \sigma_q \alpha_q^{-1} \| a_q \|$,
we may now express $B^{(k)}$ as
\begin{align*}
B^{(k)}
   &= \E_{\sigma,\tau} \Bigl[
   \sum_{s \leqslant k} \sum_{t \leqslant k}
\tilde{a}_s \tilde{a}_t
\Bigl( \sum_{r \leqslant s} \tau_r \alpha_r b^{(s)}_r \Bigr)
\Bigl( \sum_{r \leqslant t} \tau_r \alpha_r b^{(t)}_r \Bigr)^\top
\Bigr]
\\ &= \E_{\sigma,\tau} \Bigl[
\Bigl( \sum_{s \leqslant k} \tilde{a}_s \bigl( \sum_{r \leqslant s} \tau_r \alpha_r b^{(s)}_r \bigr) \Bigr)
\Bigl( \sum_{t \leqslant k} \tilde{a}_t \bigl( \sum_{r \leqslant t} \tau_r \alpha_r b^{(t)}_r \bigr)^\top \Bigr)
\Bigr]
\\ &=
\E_{\sigma,\tau} [ m_k m_k^\top ],
\end{align*}
the expectation of a rank-one estimator given by the outer product of the vector
\[
m_k \triangleq \sum_{s \leqslant k} \tilde{a}_s
\Bigl( \sum_{r \leqslant s} \tau_r \alpha_r b^{(s)}_r \Bigr)
\]
with itself. As this vector has zero mean, $B^{(k)}$ is its covariance.

\newcommand{\htildeacc}{\gamma_t \dhtdhs{t}{t-1} \htilde_{t-1}}
\newcommand{\nutildeacc}{\tilde{\nu}_{t-1}}
\newcommand{\htildenew}{\tau_t \beta_t \dhtdzt{t} \nu_t}
\newcommand{\nutildenew}{\nu_t}

The scalar $\tilde{a}_s$ is readily accumulated online,
but the vector $\sum_{r \leqslant s} \tau_q \alpha_r b^{(s)}_r$
requires approximate forward differentiation.
We can estimate $m_k$ by
\[
\tilde{m}_k \triangleq \sum_{s \leqslant k} \tilde{a}_s
\Bigl( \sum_{r \leqslant s} \tau_r \alpha_r b^{(s)\top}_r \nu_r \Bigr)
\Bigl( \sum_{r \leqslant s} \nu_r \Bigr)
\]
which can be computed efficiently according to the recursions
\begin{align*}
   \tilde{a}_t &= \gamma_t^{-1} \tilde{a}_{t-1} + \sigma_t \beta_t^{-1} \| a_t \|
\\ \htilde_t &= \eta_t \htildeacc + \zeta_t \htildenew
\\ \tilde{\nu}_t &= \eta_t^{-1} \nutildeacc + \zeta_t^{-1} \nutildenew
\\ \tilde{m}_t &= \tilde{m}_{k-1} + \tilde{a}_t \dltdht{t} \htilde_t \tilde{\nu}_t.
\end{align*}
The coefficients $\eta_t, \zeta_t$ can be used to reduce the variance of
$\htilde_t \tilde{\nu}_t^\top$,
e.g. by the \gir choice
$\eta_t^2 = \inlinefrac{\norm{\nutildeacc}}{\norm{\htildeacc}},
\zeta_t^2 = \inlinefrac{\norm{\nutildenew}}{\norm{\htildenew}}$.

Although $\E_\nu [\tilde{m}_k] = m_k$ (i.e. $\tilde{m}_k$ is an unbiased estimator of $m_k$),
$\E_\nu [\tilde{m}_k \tilde{m}_k^\top] \neq \E_\nu [\tilde{m}_k] \E_\nu [\tilde{m}_k^\top]$
and therefore $\tilde{m}_k \tilde{m}_k^\top$ is not an unbiased estimator of $B^{(k)}$.
In order to estimate $B^{(k)}$ we require a replication $\tilde{n}_k$ of $\tilde{m}_k$ with
independent spatial noise $\mu$ in place of $\nu$:
\[
\tilde{n}_k \triangleq \sum_{s \leqslant k} \tilde{a}_s
\Bigl( \sum_{r \leqslant s} \tau_r \alpha_r b^{(s)\top}_r \mu_r \Bigr)
\Bigl( \sum_{r \leqslant s} \mu_r \Bigr)
\]
computed by similar recursions as $\tilde{m}_k$.
Now
\[
\E_{\sigma,\tau,\nu,\mu} [ \tilde{m}_k \tilde{n}_k^\top ]
=
\E_{\sigma,\tau} \bigl[ \E_\nu [\tilde{m}_k] \E_\mu [\tilde{n}_k^\top] \bigr]
=
\E_{\sigma,\tau} [m_k m_k^\top]
= B^{(k)}.
\]
It should be noted that although $B^{(k)}$ is symmetric PSD, the estimates $\tilde{m}_k \tilde{n}_k^\top$ are not.
Symmetry may however be restored by use of the estimator $\inlinefrac{1}{2} (\tilde{m}_k \tilde{n}_k^\top + \tilde{n}_k \tilde{m}_k^\top)$.

\section{Hyperparameter Settings for Variance Reduction Experiments}
\label{sec:q_experiments_hyperparameters}

The following table lists the hyperparameter settings used for the experiments in Section \ref{sec:q_experiments}:
    
    \begin{tabular}{ll|rrrr}
    \Tstrut
        $Q_0$ & $\alpha$ & Learning rate & Momentum & $\bar{B}$ decay & $\bar{B}$ dampening
    \Bstrut
    \\
        \hline
    \Tstrut
        identity & \gir & 0.005 & 0.8 & & \\
        identity & ours & 0.005 & 0.5 & & \\
        ours & \gir & 0.005 & 0.5 & 0.9 & 0.008 \\
        ours & ours & 0.003 & 0.8 & 0.9 & 0.005 \\
    \Bstrut
    \end{tabular}

These settings were found by grid search on
learning rate in $\{0.001,0.003,0.005,0.007,0.009\}$,
momentum in $\{0.5,0.8\}$,
$\bar{B}$ decay rate in $\{0.8, 0.9, 0.95\}$ and
$\bar{B}$ dampening coefficient in $\{5 \times 10^{-3}, 5 \times 10^{-4}, 5 \times 10^{-5}\}$.

\section{Variance of Preactivation-Space Projection} \label{sec:preuoro_variance_comparison}

This appendix explores the variance of the estimator from Section \ref{sec:preuoro}
and its relationship to the variance of the usual total gradient estimator from Section \ref{sec:total_gradient_estimate} (Equation \ref{sec:total_gradient_estimate}).
The former is given by Equation \ref{eqn:preuoro_total_gradient_estimate}:
\begin{align*}
\vectorized \Bigl( \sum_{t \leqslant T}
\bar{B}^{(t) \top} \bar{Q} \tau \tau^{\top} \bar{Q}^{-1} \bar{S}^{(t)} \bar{J}
\Bigr)
\end{align*}
with matrices $
\bar{B}^{(t) \top} = \begin{pmatrix}
b_1^{\smash[t]{(t)}} & \cdots & b_T^{\smash[t]{(t)}}
\end{pmatrix},
\bar{Q} = \diag(\alpha),
\bar{S}^{(t)}_{i j} = \delta_{i j} \indicator{i \geqslant t},
\bar{J} = \begin{pmatrix}
a_1 & \cdots & a_T
\end{pmatrix}^\top
$
that mirror similarly-named quantities from Section \ref{sec:total_gradient_estimate}.

Noting that for a matrix $X$,
\begin{align*}
\trace \bigl(
\Var [ \vectorized (X) ]
\bigr)
=\; &
\E \bigl[ \| \vectorized (X) \|^2 \bigr]
-
\bigl\| \E [ \vectorized (X) ] \bigr\|^2
\\=\; &
\E \bigl[ \| X \|^2_F \bigr]
-
\bigl\| \E [ X ] \bigr\|^2_F
=
\trace \bigl( \E [ X X^\top ] \bigr)
-
\bigl\| \E [ X ] \bigr\|^2_F
\end{align*}
we have by Proposition \ref{prop:secondmoment} (with $\kappa = 0$) that
\begin{align}
& \trace \biggl(
\Var \Bigl[ \vectorized \Bigl(
\sum_{t \leqslant T} \bar{B}^{(t)\top} \bar{Q} \tau \tau^\top \bar{Q}^{-1} \bar{S}^{(t)} \bar{J}
\Bigr) \Bigr]
\biggr) \notag
\\&=
\sum_{s \leqslant T} \sum_{t \leqslant T}
\trace \bigl(
\E [
\bar{J}^\top \bar{S}^{(s)} \bar{Q}^{-\top} \tau \tau^\top \bar{Q}^\top \bar{B}^{(s)}
\bar{B}^{(t)\top} \bar{Q} \tau \tau^\top \bar{Q}^{-1} \bar{S}^{(t)} \bar{J}
]
\bigr)
-
\| \dldw \|^2_F \notag
\\&=
\sum_{s \leqslant T} \sum_{t \leqslant T}
\trace ( \bar{B}^{(s)} \bar{B}^{(t)\top} \bar{Q} \bar{Q}^\top )
\trace ( \bar{S}^{(t)} \bar{J} \bar{J}^\top \bar{S}^{(s)} (\bar{Q} \bar{Q}^\top)^{-1} )
+
\| \dldw \|^2_F. \label{eqn:totalvariance_preuoro}
\end{align}

The variance of the usual estimator (Section \ref{sec:total_gradient_estimate}) is given by Equation \ref{eqn:total_gradient_variance_trace}:
\begin{align}
& \trace \Bigl( \Var \bigl[
\sum_{t \leqslant T} b^{(t)\top} Q u u^\top Q^{-1} S^{(t)} J
\bigr]
\Bigr) \notag
\\&=
\sum_{s \leqslant T} \sum_{t \leqslant T}
\trace ( b^{(s)} b^{(t)\top} Q Q^\top )
\trace ( S^{(t)} J J^\top S^{(s)} (Q Q^\top)^{-1} )
+
\| \dldw \|^2_F.  \label{eqn:totalvariance_recap}
\end{align}
We will now relate Equation \ref{eqn:totalvariance_recap} to Equation \ref{eqn:totalvariance_preuoro}
using the following observations:
\begin{align*}
b^{(t)\top} = \vectorized (\bar{B}^{(t)}),
Q = \bar{Q} \otimes Q_0,
S^{(t)} = \bar{S}^{(t)} \otimes I, \mbox{ and }
J J^\top = \bar{J} \bar{J}^\top \otimes I.
\end{align*}
Thus $\trace ( b^{(s)} b^{(t)\top} Q Q^\top )$ can be written
\begin{align*}
\trace ( b^{(s)} b^{(t)\top} Q Q^\top )
&=
\vectorized ( \bar{B}^{(t)} )^\top (\bar{Q} \bar{Q}^\top \otimes Q_0 Q_0^\top) \vectorized (\bar{B}^{(s)})
\\&=
\vectorized ( \bar{B}^{(t)} )^\top \vectorized ( Q_0 Q_0^\top \bar{B}^{(s)\top} \bar{Q} \bar{Q}^\top )
\\&=
\trace ( \bar{B}^{(t)} Q_0 Q_0^\top \bar{B}^{(s)\top} \bar{Q} \bar{Q}^\top )
\end{align*}
and $\trace ( S^{(t)} J J^\top S^{(s)} (Q Q^\top)^{-1} )$ can be written
\begin{align*}
\trace ( S^{(t)} J J^\top S^{(s)} (Q Q^\top)^{-1} )
&=
\trace \bigl(
  (\bar{S}^{(t)} \otimes I)
  (\bar{J} \bar{J}^\top \otimes I)
  (\bar{S}^{(s)} \otimes I)
  (\bar{Q}^{-2} \otimes (Q_0 Q_0^\top)^{-1})
  \bigr)
\\&=
\trace \bigl( \bar{S}^{(t)} \bar{J} \bar{J}^\top \bar{S}^{(s)} (\bar{Q} \bar{Q}^\top)^{-1}) \otimes (Q_0 Q_0^\top)^{-1} \bigr)
\\&=
\trace \bigl( \bar{S}^{(t)} \bar{J} \bar{J}^\top \bar{S}^{(s)} (\bar{Q} \bar{Q}^\top)^{-1} \bigr) \trace \bigl( (Q_0 Q_0^\top)^{-1} \bigr).
\end{align*}
This results in the following expression for the dominant term of the variance in Equation \ref{eqn:totalvariance_recap}:
\[
\sum_{s \leqslant T} \sum_{t \leqslant T}
\trace \bigl( \bar{B}^{(t)} Q_0 Q_0^\top \bar{B}^{(s)\top} \bar{Q} \bar{Q}^\top \bigr)
\trace \bigl( \bar{S}^{(t)} \bar{J} \bar{J}^\top \bar{S}^{(s)} (\bar{Q} \bar{Q}^\top)^{-1} \bigr) \trace \bigl( (Q_0 Q_0^\top)^{-1} \bigr).
\]

\bibliography{main}

\end{document}